\documentclass[10pt,twocolumn,letterpaper]{article}

\usepackage[pagenumbers]{cvpr}      

\usepackage{epsfig}
\usepackage{graphicx}
\usepackage{amsmath}
\usepackage{amssymb}
\usepackage{booktabs}
\usepackage{stfloats}
\usepackage{url}            
\usepackage{amsfonts}       
\usepackage{nicefrac}       
\usepackage{bbding}
\usepackage{microtype}
\usepackage{multirow}
\usepackage{subcaption}
\usepackage{threeparttable}
\usepackage{fancyhdr}
\usepackage{array}
\usepackage{diagbox}
\usepackage{algorithm}
\usepackage{algorithmic}
\usepackage{enumerate}
\usepackage{makecell}
\usepackage{adjustbox}
\usepackage{dsfont}
\usepackage{wrapfig}
\usepackage{xspace}
\usepackage{enumitem}
\usepackage{eufrak}
\usepackage{tikz}
\usepackage{xcolor}

\newtheorem{definition}{Definition}

\newtheorem{proposition}{Proposition}
\newtheorem{proof}{Proof}

\newcommand{\SysName}{\texttt{SiRIIB}\xspace}
\definecolor{cvprblue}{rgb}{0.21,0.49,0.74}
\usepackage[pagebackref,breaklinks,colorlinks,citecolor=cvprblue]{hyperref}

\title{Singular Regularization with Information Bottleneck Improves Model's Adversarial Robustness}

\author{
Guanlin Li \\
Nanyang Technological University, S-Lab\\
{\tt\small guanlin001@e.ntu.edu.sg}
\and
Naishan Zheng\\
Nanyang Technological University, S-Lab
\and
Man Zhou\\
Nanyang Technological University, S-Lab
\and
Jie Zhang\\
Nanyang Technological University
\and
Tianwei Zhang\\
Nanyang Technological University
}

\begin{document}
\maketitle

\begin{abstract}
Adversarial examples are one of the most severe threats to deep learning models. Numerous works have been proposed to study and defend adversarial examples. However, these works lack analysis of adversarial information or perturbation, which cannot reveal the mystery of adversarial examples and lose proper interpretation. In this paper, we aim to fill this gap by studying adversarial information as unstructured noise, which does not have a clear pattern. Specifically, we provide some empirical studies with singular value decomposition, by decomposing images into several matrices, to analyze adversarial information for different attacks. Based on the analysis, we propose a new module to regularize adversarial information and combine information bottleneck theory, which is proposed to theoretically restrict intermediate representations. Therefore, our method is interpretable. Moreover, the fashion of our design is a novel principle that is general and unified. Equipped with our new module, we evaluate two popular model structures on two mainstream datasets with various adversarial attacks. The results indicate that the improvement in robust accuracy is significant. On the other hand, we prove that our method is efficient with only a few additional parameters and able to be explained under regional faithfulness analysis.
\end{abstract}

\section{Introduction}

Deep learning models, especially Convolutional Neural Networks (CNNs) in computer vision, suffer from adversarial examples~\cite{goodfellow_explaining_2015}, which can cause the models to give incorrect responses. With deep learning models becoming an important part of various services, it is vital to study and understand adversarial examples. To the best of our knowledge, there are three main perspectives for studying and understanding adversarial examples. The first is to treat adversarial examples as some specific features~\cite{ilyas_adversarial_2019}. These features help the model to learn a robust representation to defend against adversarial attacks. Other data without such features will only make the model obtain clean accuracy but fail under attacks. The second perspective is to study adversarial properties over the whole training set~\cite{wang_improving_2020,ge_advancing_2023}. They find that different training data have unique adversarial properties. Some data are mainly contributing to the clean accuracy. Some contribute primarily to adversarial robustness. The last perspective is to study how the model learns adversarial examples during the training process~\cite{rice_overfitting_2020,dong_exploring_2022}. Because models are much easier to overfit adversarial examples, studying and understanding such a phenomenon will help us design better training methods. Although previous works have studied adversarial examples from several perspectives, they mainly focus on the training process and study one specific attack, i.e., PGD attack~\cite{madry_towards_2018}. That will lose some guarantees of consistency between training and testing and among different attacks. Different from previous works, we take one step forward to study the straightforward properties of adversarial examples under different attacks and training strategies to make our observations consistent and general. 

\begin{figure}[t]
\centering
\includegraphics[width=0.9\linewidth]{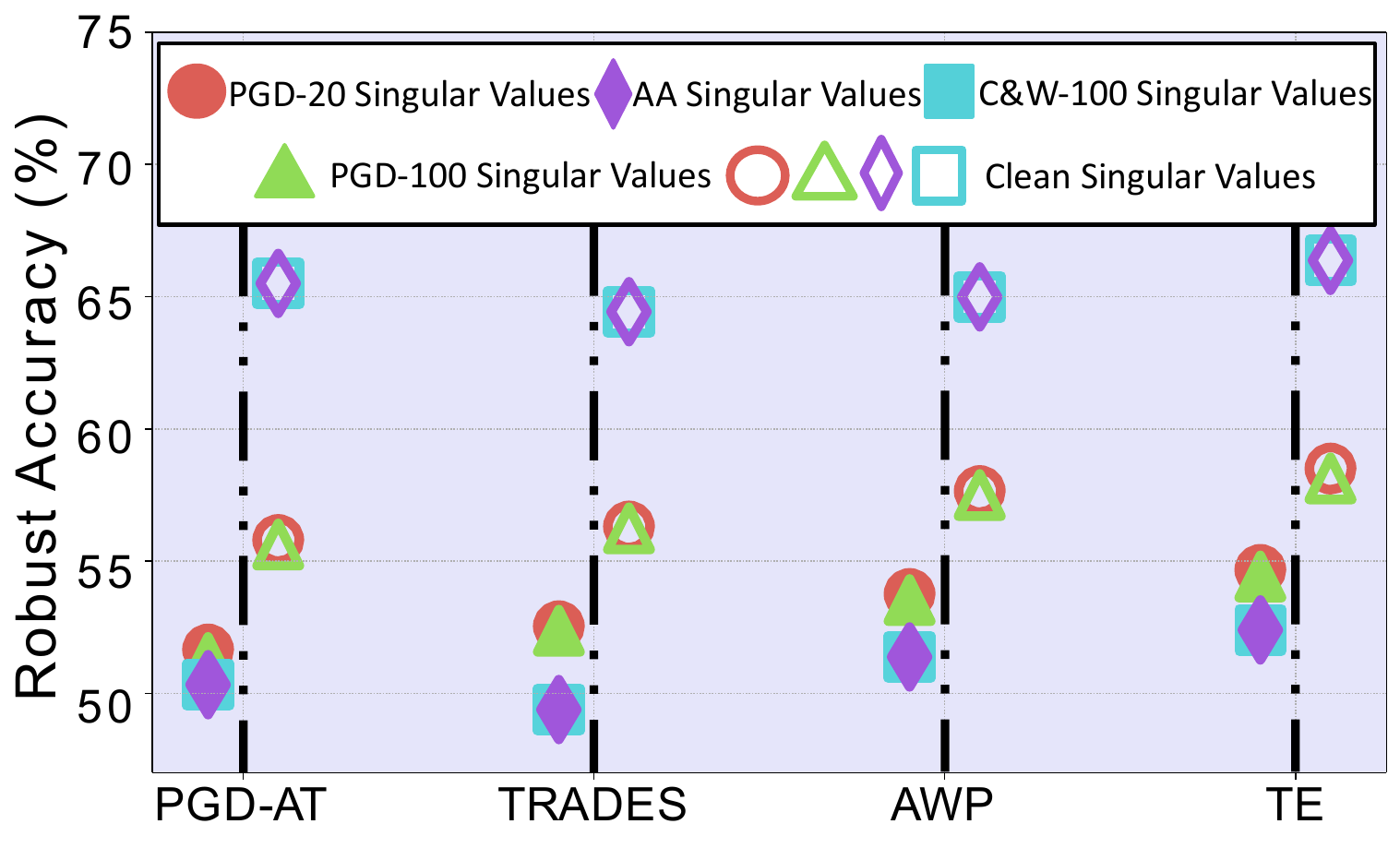}
\caption{Combination between singular vectors of adversarial examples with singular values from clean images and adversarial examples under different attacks for different training strategies (PGD-AT~\cite{madry_towards_2018}, TRADES~\cite{zhang_theoretically_2019}, AWP~\cite{wu_adversarial_2020}, and TE~\cite{dong_exploring_2022}). Solid markers (left) mean we do not replace the singular values with clean ones. Hollow markers (right) mean we replace the singular values with clean ones for specific attacks.}
\label{fig:intro} 
\vspace{-13pt}
\end{figure}

In novelty, we analyze adversarial examples from the matrix decomposition aspect and find that different adversarial examples from different sources have similar properties. Considering that the given image is made up of matrices stacked along the color channel, we can decompose each matrix for every color channel to obtain the singular values and corresponding singular vectors, dubbed singular value decomposition (SVD)~\cite{wall_singular_2003}. Generally speaking, the singular values and vectors store all the information of the given image. In the image restoration perspective, different types of corruption can be reflected in singular values or vectors~\cite{zhang_decomposition_2023}. Similarly, added perturbation in the adversarial examples can be seen as unstructured noise, i.e., one specific type of corruption, which is related to the model's parameters and the clean images and does not have a specific pattern. Naturally, we can do SVD on adversarial examples and obtain their singular values and vectors to study the properties of adversarial examples with these important elements. Specifically, we study singular values and vectors of adversarial examples and their corresponding clean data by performing SVD on them, respectively. Then, we combine the singular vectors of adversarial examples and singular values of the corresponding clean data. In this way, we can study and compare adversarial information in singular values and vectors. In Figure~\ref{fig:intro}, we observe that after using clean singular values, the robust accuracy will increase under different attacks. Our results indicate that, for different training strategies and adversarial attacks, singular values consistently contain adversarial information. Therefore, replacing them with clean singular values will increase the robust accuracy. The detailed results and analysis can be found in Section~\ref{sec:pre}.

This consistency among training methods and various attacks inspires us to perform calibration on adversarial information in the singular values and vectors to defend against adversarial examples. However, it is not trivial to directly modulate the singular values and vectors for given images, because they are hidden behind the pixel data. Therefore, we introduce and develop custom modules from the previous work~\cite{zhang_decomposition_2023} to separate singular values and vectors from the images and perform calibration on them, respectively. To make our method more explainable, we integrate the singular regularization operation in a new plug-in module based on information bottleneck theory~\cite{alemi_deep_2017}, dubbed Singular Regularization with Implicit Information Bottleneck (\SysName). The information bottleneck theory can be adopted to explain intermediate features of the models and give a guarantee for robustness. Specifically, \SysName can calibrate adversarial examples from the perspective of SVD and transfer extracted features to the main classification model following the information bottleneck theory, i.e., to compress the modulated information that is extracted from inputs, and to contain sufficient information to guide the following prediction task. On the other hand, \SysName is a general module that can cooperate with residual-based CNNs, such as ResNet~\cite{he_deep_2016} and WideResNet~\cite{zagoruyko_wide_2016}, and different training strategies, such as PGD-AT~\cite{madry_towards_2018} and TRADES~\cite{zhang_theoretically_2019}. With our comprehensive evaluation, we prove that \SysName with few parameters will improve robust accuracy under various attacks. Overall, our contribution can be summarized as follows:
\begin{itemize}
    \item We study the properties of adversarial examples with SVD and find different attacks that will always influence the singular values and vectors but to different degrees.
    \item We design a new plug-in module \SysName based on information bottleneck theory with singular regularization, which is proven to be effective in defending against various attacks. Our \SysName is universal and can cooperate with different models. 
    \item We show that \SysName can calibrate adversarial information in adversarial examples and induce the model to learn a more compressed representation. Furthermore, it can reduce the sensitivity of local perturbation during the attack.
\end{itemize}

\section{Related Works}

To improve the adversarial robustness of models, there are several methods from different perspectives. In summary, there are two main traces, i.e., studying training methods~\cite{madry_towards_2018} and studying models~\cite{xie_feature_2019}.

\textbf{Training Methods}. When training the classification models, we can generate adversarial examples for each batch of clean data and train the models with these adversarial examples~\cite{madry_towards_2018}, which is called adversarial training. It can be formulated as the following min-max problem:
\begin{align*}
    \min_\theta \max_{x_\mathrm{adv}}L(x_\mathrm{adv}, y;\theta)
\end{align*}
where $x_\mathrm{adv}$ is the training sample generated from a clean one $x$ to maximize the loss function $L(\cdot)$, $y$ is the ground-truth label, $\theta$ is the model parameters. It is to say, we find $x_\mathrm{adv}$ based on $x$ to maximize the value $L(x_\mathrm{adv}, y;\theta)$, and then we optimize the model to minimize the loss $L(x_\mathrm{adv}, y;\theta)$. Zhang et al.~\cite{zhang_theoretically_2019} explore the trade-off between clean accuracy and robust accuracy during adversarial training and propose a new training strategy to better balance them. 


\textbf{Model Perspective}. There are some works aiming to improve the existing model's robustness by replacing modules with newly designed ones. For example, Huang et al.~\cite{huang} propose a new residual block focusing on the depth and width to improve the residual model's robustness. On the other hand, some works~\cite{bai} focus on comparing the robustness of traditional convolutional models and visual transformers~\cite{vit}.


Different from previous works, we aim to introduce a new plug-in module, whose design is inspired by the singular value decomposition~\cite{wall_singular_2003} and information bottleneck theory~\cite{alemi_deep_2017}. \textbf{Our module is deterministic, differentiable, explainable, and efficient, with only a few parameters.}

\begin{figure}[t]
\centering
\includegraphics[width=0.95\linewidth]{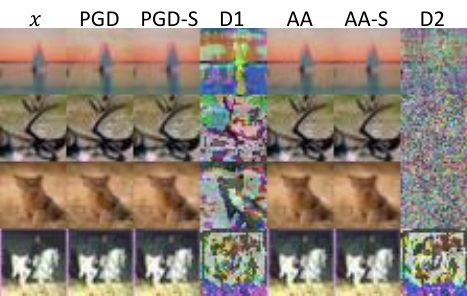}
\caption{Visualization results of singular value combinations. $x$ is clean images. PGD and AA mean adversarial examples generated by PGD-20 and AA. PGD-S and AA-S mean we replace singular values of adversarial examples with clean singular values. D1 and D2 are differences between PGD-S and PGD and between AA-S and AA, respectively.}
\label{fig:intro_svd} 
\vspace{-15pt}
\end{figure}

\section{Preliminaries}
\label{sec:pre}

In this section, we first provide some empirical studies to analyze adversarial examples under SVD. Then, we introduce a general method to modulate the singular values and vectors for a given image without explicitly performing SVD on it.

\subsection{Empirical Study under SVD}

In our empirical studies, to understand singular values and vectors of adversarial examples and their corresponding clean data, we first perform SVD on them, respectively. Then, we combine the singular vectors of adversarial examples and singular values of corresponding clean data. In Figure~\ref{fig:intro}, we compare the robust accuracy of ResNet-18~\cite{he_deep_2016} trained with different adversarial training methods on CIFAR-10 from the aspects of singular values. The results indicate a strong consistency, which is not related to the training method but related to the attacks. First, for all attacks, after using clean singular values, robust accuracy will increase. On the other hand, different attacks have slight differences under the view of adversarial information. Specifically, based on our observations, some adversarial attacks, such as PGD~\cite{madry_towards_2018}, hide the adversarial information in the singular values and vectors equally. Because, if we replace the singular values of adversarial examples with the singular values of their corresponding clean images, the robust accuracy will slightly increase. However, some adversarial attacks, like C\&W~\cite{carlini_towards_2017} and Autoattack (AA)~\cite{croce_reliable_2020}, mainly hide adversarial information in singular values. Simply replacing the singular values of adversarial examples with the singular values of their corresponding clean images will significantly increase robust accuracy. In Figure~\ref{fig:intro_svd}, we further visualize the differences before and after replacing with clean singular values for adversarial examples generated by PGD-20 and AA. It is clear that clean singular values will cause a specific noise pattern for PGD-20 examples and part of AA examples. But for some AA examples, replacing with clean singular values will cause a random noise pattern, which is closer to the adversarial perturbation. This explains why introducing clean singular values in adversarial examples generated by AA leads to higher accuracy improvements. These observations inspire us to calibrate singular values and vectors with clean images to defend against adversarial attacks. However, in practice, we cannot manually or randomly modify the SVD results of the given images and then reconstruct the images based on the modified SVD results, because such operations will be non-differentiable and random, which will cause gradient obfuscation, which is proven to give a false sense of security~\cite{athalye_obfuscated_2018}. In the following, we will introduce a deterministic, differentiable, and learnable method to modulate singular values and vectors.

\subsection{Singular Regularization}

To design a deterministic, differentiable, and learnable method to modulate singular values and vectors, we follow the previous method~\cite{zhang_decomposition_2023} to implicitly modify the singular values and vectors from the images, respectively. It is because singular values and vectors correspond to different information, i.e., singular vectors represent the directions in the image space along which the most variance occurs, and singular values represent the importance of each corresponding singular vector. Therefore, we calibrate them separately, i.e., singular value modification and singular vector modification, and we fuse the modified results to obtain modulated features. 

\textbf{Singular Vector Modification.} During this operation, the most important thing is to keep the singular values unchanged. Intuitively, we can find the fact that multiplying an arbitrary matrix by orthogonal matrices will not change the singular values of the arbitrary matrix. For example, given an arbitrary real matrix $X\in \mathbb{R}^{n\times m}$ and an orthogonal real matrix $P\in \mathbb{R}^{n\times n}$, the singular values of $PX$ are the same as those of $X$. To prove it, we can first write $X$ into SVD format, $X=U\Sigma V'$, where $U$ and $V'$ are two orthogonal matrices. Then, we have $PX=PU\Sigma V'$, in which we only need to prove $(PU)(PU)'=I$, where $I$ is the identity matrix. Clearly, we have
\begin{align*}
  (PU)(PU)' = PUU'P' = PIP' = PP' = I.  
\end{align*} 
Therefore, this operation will only affect the singular vectors. Based on the above information, we can restrict the learnable weight matrix to be orthogonal and use it to modify the singular vector for given images.

\textbf{Singular Value Modification.} On the other hand, modifying singular values is non-trivial, due to the inherent inaccessibility of the singular values. However, the previous work~\cite{zhang_decomposition_2023} finds that the singular values can be modified alone in the Fourier domain. Specifically, we can present $X\in \mathbb{R}^{n\times m}$ in two formations, i.e., 
\begin{align*}
    &X=U\Sigma V',\\
    X=\frac{1}{nm}\sum_{u=0}^{n-1}&\sum_{v=0}^{m-1}G(u,v)e^{2\pi i(\frac{ua}{n}+\frac{vb}{m})}, \\
    a\in &\mathbb{R}^{n-1}, b\in \mathbb{R}^{m-1}
\end{align*}
where $G(u,v)$ represents the coefficients of the Fourier transform of $X$. If we use $u_i$ and $v_i$ to present the columns of $U$ and $V$, and use $\sigma_i$ to represent singular values in $\Sigma$, we can rewrite $X$ in a way of coefficients multiplying a group of basis, i.e., 
\begin{align*}
  X=\sum_{i=1}^{min(n,m)}\sigma_i u_i v_i' = \frac{1}{nm}\sum_{u=0}^{n-1}\sum_{v=0}^{m-1}G(u,v)\phi(u,v),  
\end{align*} 
where $\phi(u,v)=e^{2\pi i(\frac{ua}{n}+\frac{vb}{m})}$. In the Fourier domain, we can modify the coefficients $G(u,v)$, and the previous work~\cite{zhang_decomposition_2023} proves that $\sigma_i$ and $G(u,v)$ can be approximated to each other. Therefore, we transfer $X$ into the Fourier domain and modify $G(u,v)$ as an equivalent operation of modifying singular values.

\section{Singular Regularization with Implicit Information Bottleneck}

In this section, we first introduce the information bottleneck theory to the module design and prove that information compression and regularization can be implicitly achieved with carefully designed additional connections. Then, we propose \SysName, which contains the singular regularization function and follows the information bottleneck theory, including how it cooperates with residual-based models and how it can be trained during the adversarial training process.

\begin{figure}[t]
\centering
\includegraphics[width=1.0\linewidth]{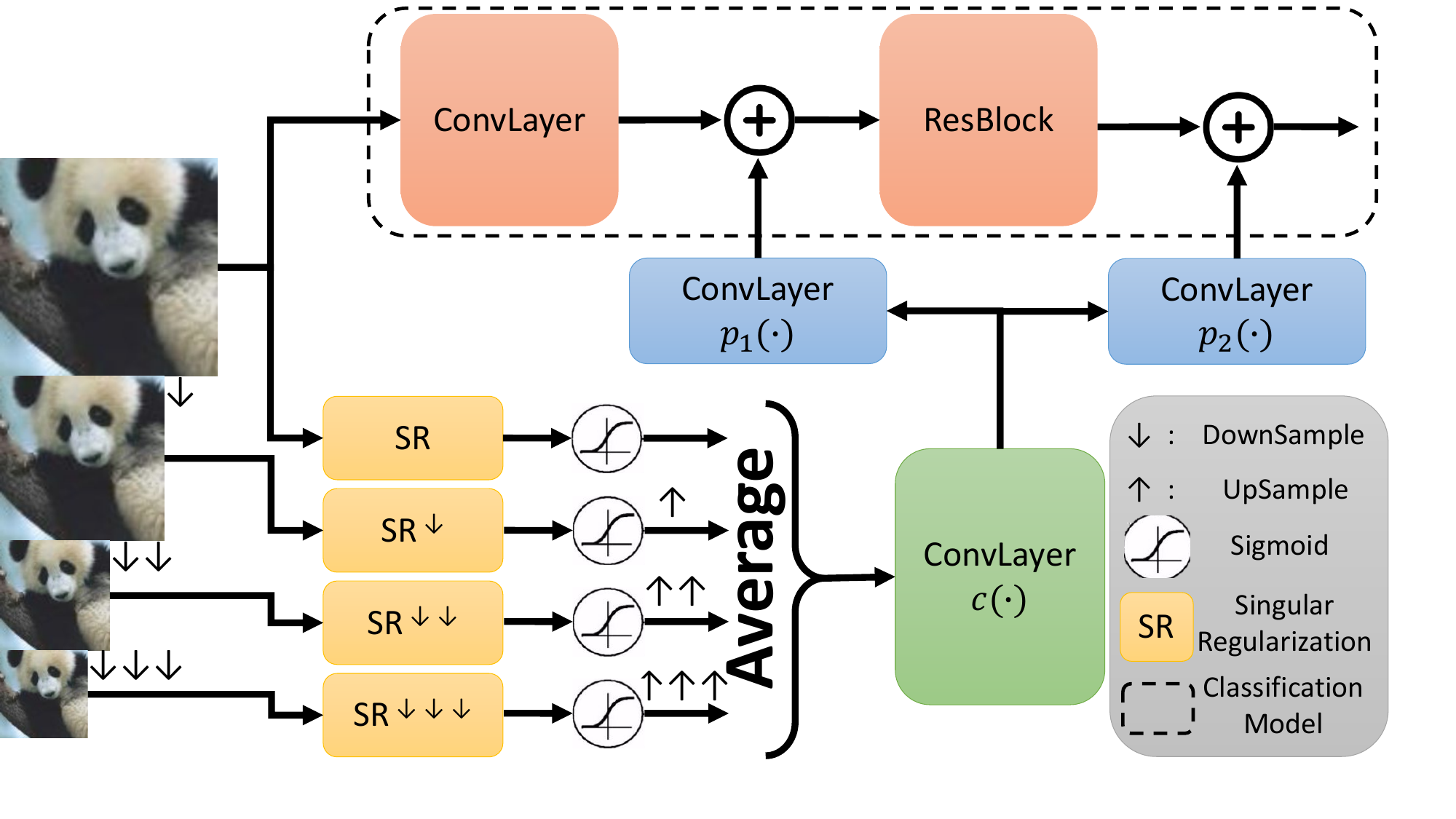}
\caption{Overview of \SysName. We adopt multiple scale inputs for the singular regularization modules. Then, after the Sigmoid function, we average the outputs to obtain smoothed results. Finally, we adopt several linear transformations to modify the feature dimensions and add the outputs to the original classification model.}
\label{fig:overview} 
\vspace{-10pt}
\end{figure}

\subsection{Information Bottleneck in Module Design}

The information bottleneck (IB) theory~\cite{alemi_deep_2017} can be adopted to explain the model's adversarial robustness. In previous work, it usually needs to be explicitly optimized with Kullback–Leibler divergence. However, we find that it can be implicitly executed with carefully designed skip connections. First, we recall the details of the information bottleneck theory. 

For a deep learning model, the mutual information between the intermediate representation and the target label is defined as follows.
\begin{definition}\label{def:mi}
Given a deep learning model $f(\cdot; \Theta)$, parameterized by $\Theta$, we regard some intermediate layers that project the input $x$ to an encoding $z$, defined as $p(z|x;\theta), \theta\in\Theta$. Suppose that the target label for $x$ is $y$, and the mutual information between $z$ and $y$ is defined as $I(z,y;\theta)$.
\end{definition}
Similarly, based on Definition~\ref{def:mi}, we have mutual information between two random variables $Z$ and $Y$, i.e., 
\begin{align*}
    I(Z, Y;\theta)=\int dx dy 
 p(z,y|\theta)\log\frac{p(z,y|\theta)}{p(z|\theta)p(y|\theta)},
\end{align*}
where $z$ and $y$ are observations of variables $Z$ and $Y$. To find the best representation $Z$ for the input $X$ under the information complexity constraint, $I(X,Z)\le I_c$, we obtain the following optimization goal:
\begin{align*}
     \max_\theta I(Z,Y;\theta) \quad s.t. \quad I(X,Z;\theta)\le I_c,
\end{align*}
which can be further converted to 
\begin{align*}
   \max_\theta R(\theta) = I(Z,Y;\theta) - \lambda I(Z,X;\theta) 
\end{align*}
by introducing the Lagrange multiplier $\lambda$. By maximizing $R(\theta)$, we aim to obtain $Z$, which meets the conditions where we can predict $Y$ from $Z$ with the highest probability, and $Z$ compresses $X$ as much as possible. Specifically, a large $\lambda$ will increase the compression ratio. Clearly, maximizing $R(\theta)$ during the model training process is to explicitly optimize the latent representation of $f(\cdot;\Theta)$. 

However, without directly optimizing the $R(\theta)$, we find that it is possible to implicitly optimize the same objective with correctly designed skip connections. In the following, we give a formal definition for the IB skip connection used in our paper.

\begin{definition}\label{def:ibsc}
Given an intermediate representation $z$ of the input $x$, we say that there exists an IB skip connection between the input $x$ and the representation $z = f(x;\theta) + g(x;\tau)$, if $I(z, x; \theta, \tau) \le I(f(x;\theta), x; \theta)$.
\end{definition}

The IB skip connection $g(x;\tau)$ in Definition~\ref{def:ibsc} means it can reduce mutual information between the input $x$ and the intermediate representation $z$ by adding restrained information to the original $f(x;\theta)$. In the following, we will give an example to build such a connection.

\noindent\textbf{\textit{Intuition 1}} Double-sided saturating activation functions, such as Sigmoid or Softmax, will compress the information and explicitly achieve the function of information bottleneck.

In fact, this intuition has been proved in the previous work~\cite{lee_reducing_2021}. Furthermore, this encoding compression is not related to loss functions. Therefore, based on the intuition, we give a proposition to show that it is possible to design a skip connection, which can reduce the mutual information between intermediate features and the inputs.

\begin{proposition}\label{pro:sc}
If functional $\mathcal{G}(\cdot; \tau)$ is equipped with a double-sided saturating activation function, e.g., Sigmoid, there exists $g(\cdot;\tau)\in \mathcal{G}(\cdot; \tau)$, which can be used as an IB skip connection.
\end{proposition}

The proof of Proposition~\ref{pro:sc} is provided in the supplementary materials. In summary, we build $g(x;\tau)$ with linear transformations and Sigmoid and show that it can reduce the mutual information after adding it to the original model.

\begin{table}[ht]
\centering
\begin{adjustbox}{max width=1.0\linewidth}
\begin{tabular}{c|c|c}
 \Xhline{1.5pt}
\textbf{\# of $p_i(\cdot)$} & \textbf{Clean Accuracy} & \textbf{AA} \\ \hline
1 & 82.53 & 48.03 \\ \hline
2 & 83.31 & 47.80 \\ \hline
3 & 83.07 & \textbf{48.05} \\
 \Xhline{1.5pt}
\end{tabular}
\end{adjustbox}
\caption{Ablation study on the number of $p_i(\cdot)$. We compare models equipped with different numbers of IB skip connections.}
\vspace{-5pt}
\label{tab:layer}
\end{table}

\begin{table}[ht]
\centering
\begin{adjustbox}{max width=1.0\linewidth}
\begin{tabular}{c|c|c}
 \Xhline{1.5pt}
\textbf{$\lambda_1$} & \textbf{Clean Accuracy} & \textbf{AA} \\ \hline
1.0 & 83.07 & 48.05 \\ \hline
1.5 & 82.84 & 48.16 \\ \hline
3.0 & 82.84 & 48.21 \\ \hline
5.0 & 83.00 & 48.36 \\ \hline
20.0 & \textbf{83.27} & \textbf{48.41} \\ \hline
30.0 & 82.99 & 48.35 \\
 \Xhline{1.5pt}
\end{tabular}
\end{adjustbox}
\caption{Ablation study on $\lambda_1$.}
\label{tab:lambda}
\vspace{-10pt}
\end{table}

\begin{figure}[t]
\centering
\includegraphics[width=1.0\linewidth]{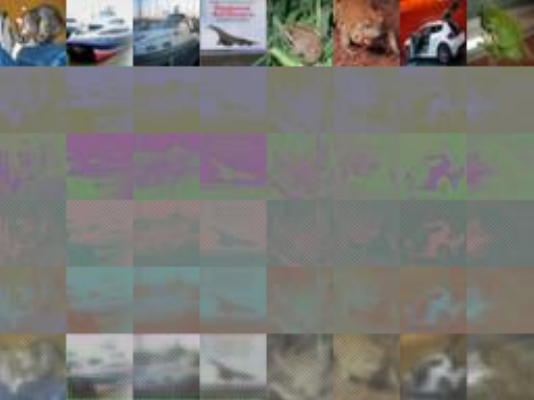}
\vspace{-15pt}
\caption{Visualization of $x_\mathrm{avg}$ under different $\lambda_1$. The first row is clean images. From the second row to the last row, $\lambda_1$ is 1.0, 1.5, 3.0, 5.0, and 20.0.}
\label{fig:lambda} 
\vspace{-15pt}
\end{figure}

\begin{figure*}[ht]
\centering
\begin{subfigure}[b]{0.23\linewidth}
\centering
\includegraphics[width=\linewidth]{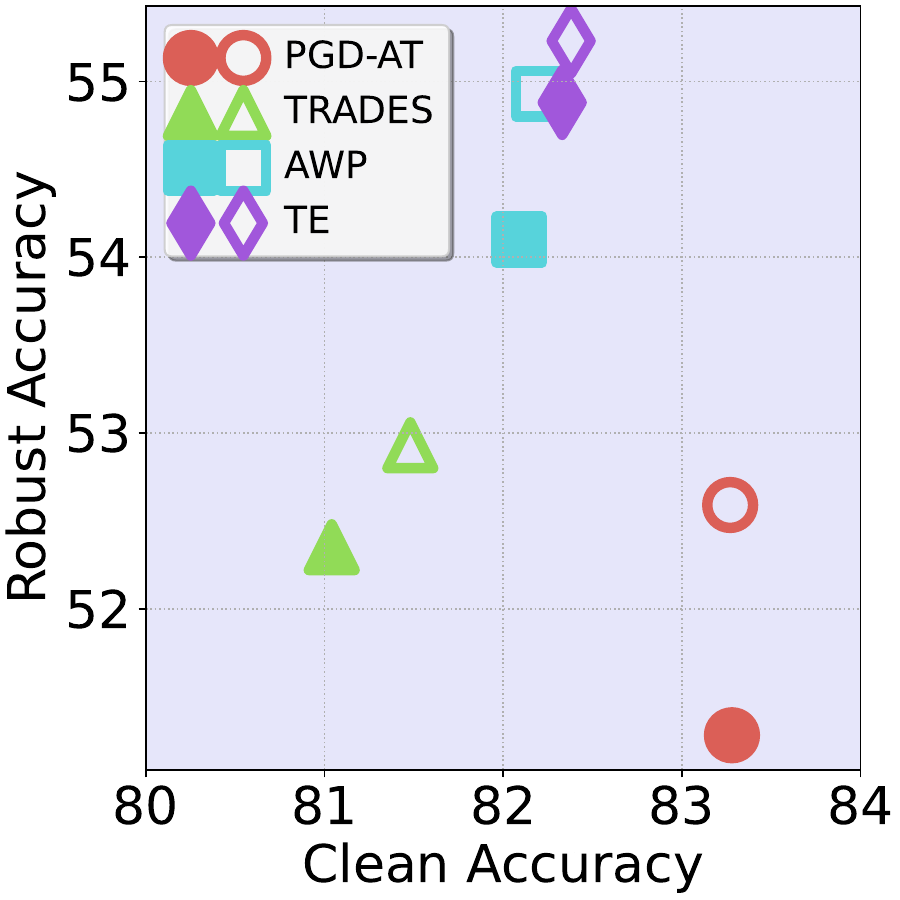}
\caption{Under PGD-20.}
\label{fig:cifar10-resnet-pgd20} 
\end{subfigure}
\begin{subfigure}[b]{0.23\linewidth}
\centering
\includegraphics[width=\linewidth]{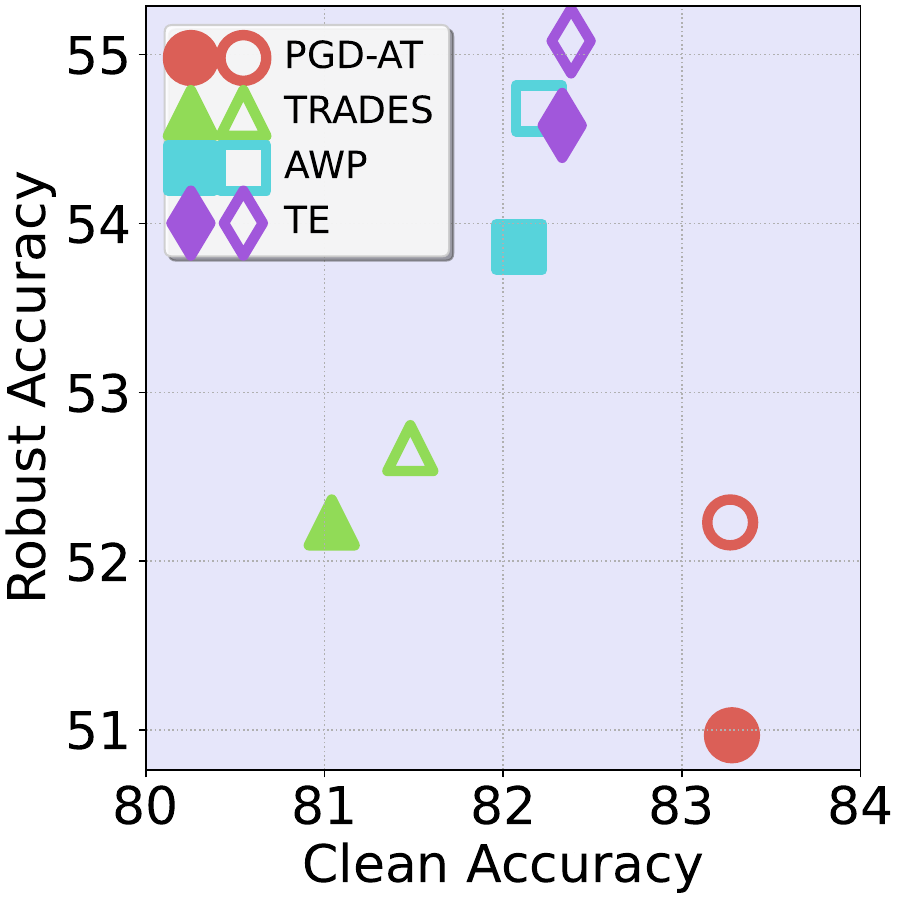}
\caption{Under PGD-100.}
\label{fig:cifar10-resnet-pgd100} 
\end{subfigure}
\begin{subfigure}[b]{0.23\linewidth}
\centering
\includegraphics[width=\linewidth]{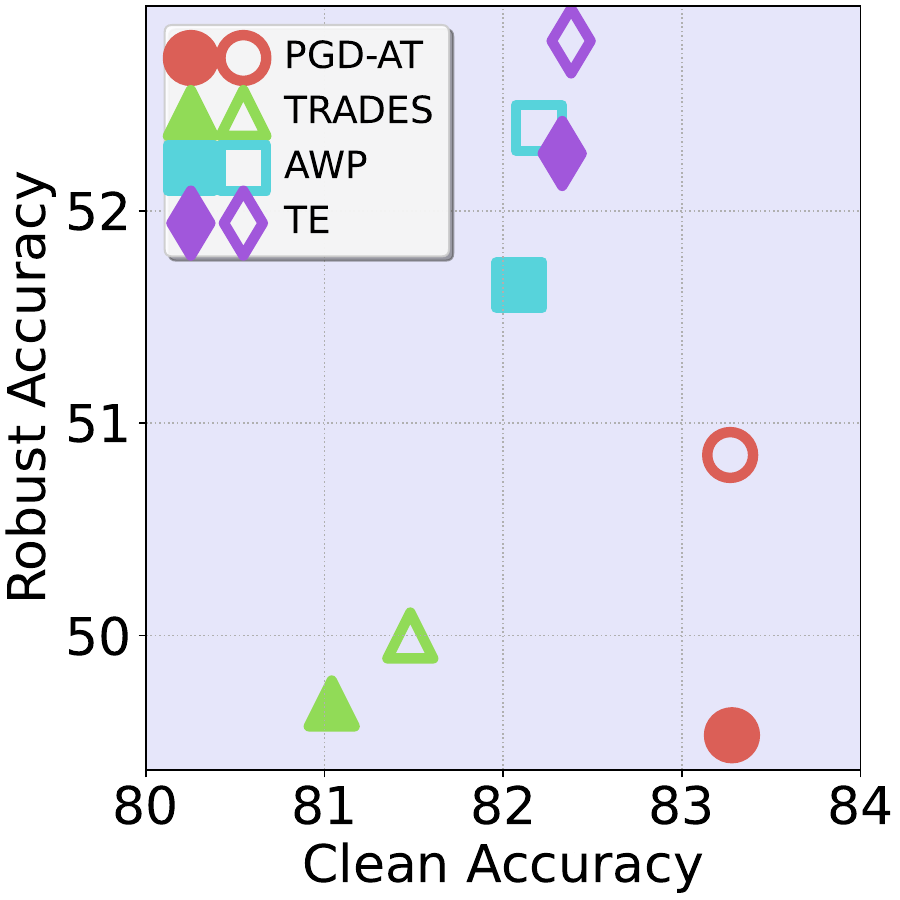}
\caption{Under C\&W-100.}
\label{fig:cifar10-resnet-cw} 
\end{subfigure}
\begin{subfigure}[b]{0.23\linewidth}
\centering
\includegraphics[width=\linewidth]{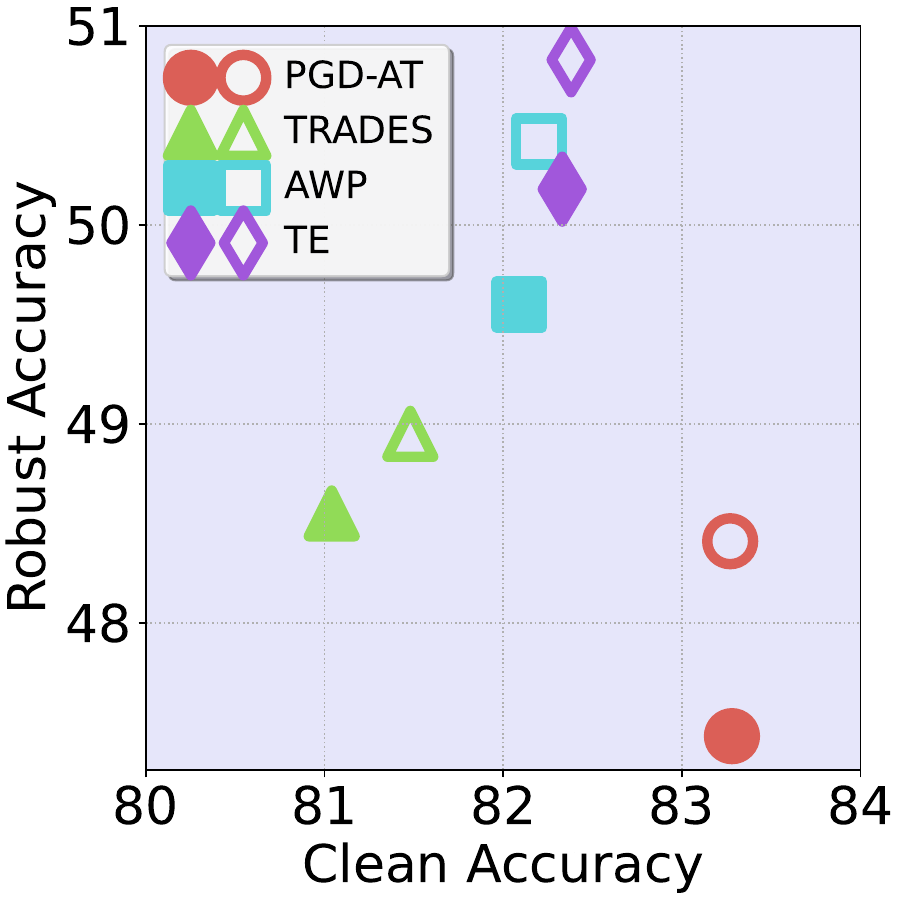}
\caption{Under AA.}
\label{fig:cifar10-resnet-aa} 
\end{subfigure}
\vspace{-10pt}
\caption{Clean accuracy and robust accuracy of ResNet-18 (Solid Markers) and ResNet-18-SR (Hollow Markers) with different training strategies on CIFAR-10. The upper right corner means that the model has the best clean and robust accuracy.}
\vspace{-5pt}
\label{fig:cifar10-resnet} 
\end{figure*}

\begin{figure*}[ht]
\centering
\begin{subfigure}[b]{0.23\linewidth}
\centering
\includegraphics[width=\linewidth]{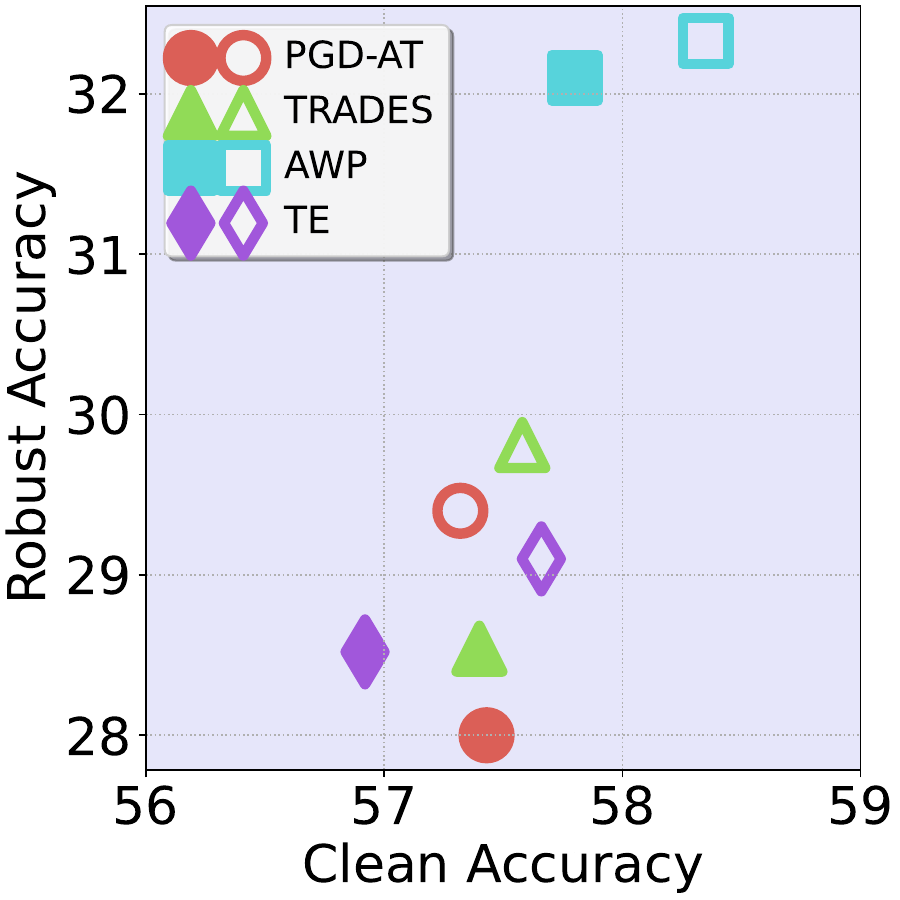}
\caption{Under PGD-20.}
\label{fig:cifar100-resnet-pgd20} 
\end{subfigure}
\begin{subfigure}[b]{0.23\linewidth}
\centering
\includegraphics[width=\linewidth]{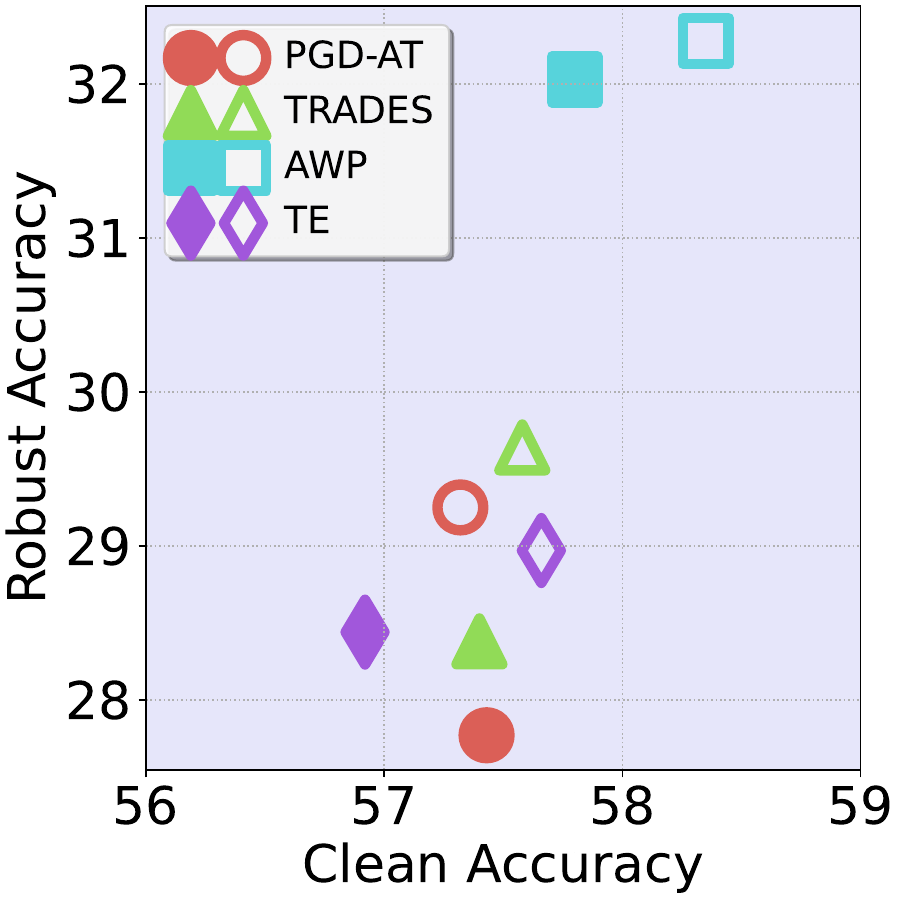}
\caption{Under PGD-100.}
\label{fig:cifar100-resnet-pgd100} 
\end{subfigure}
\begin{subfigure}[b]{0.23\linewidth}
\centering
\includegraphics[width=\linewidth]{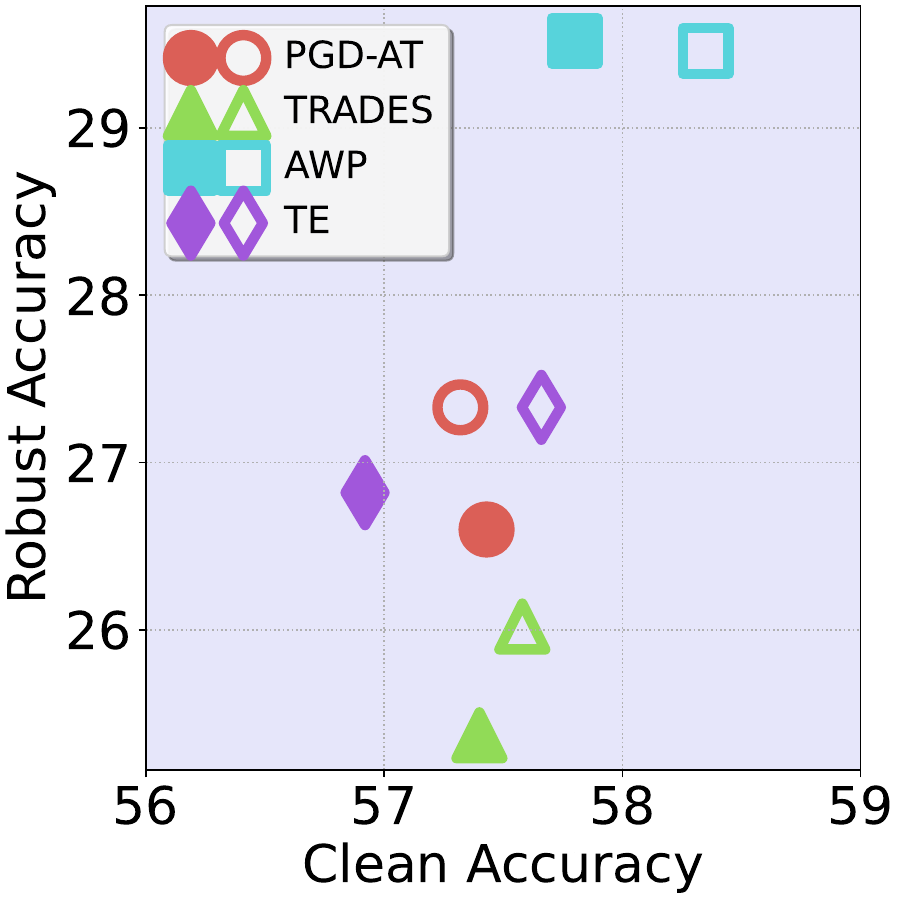}
\caption{Under C\&W-100.}
\label{fig:cifar100-resnet-cw} 
\end{subfigure}
\begin{subfigure}[b]{0.23\linewidth}
\centering
\includegraphics[width=\linewidth]{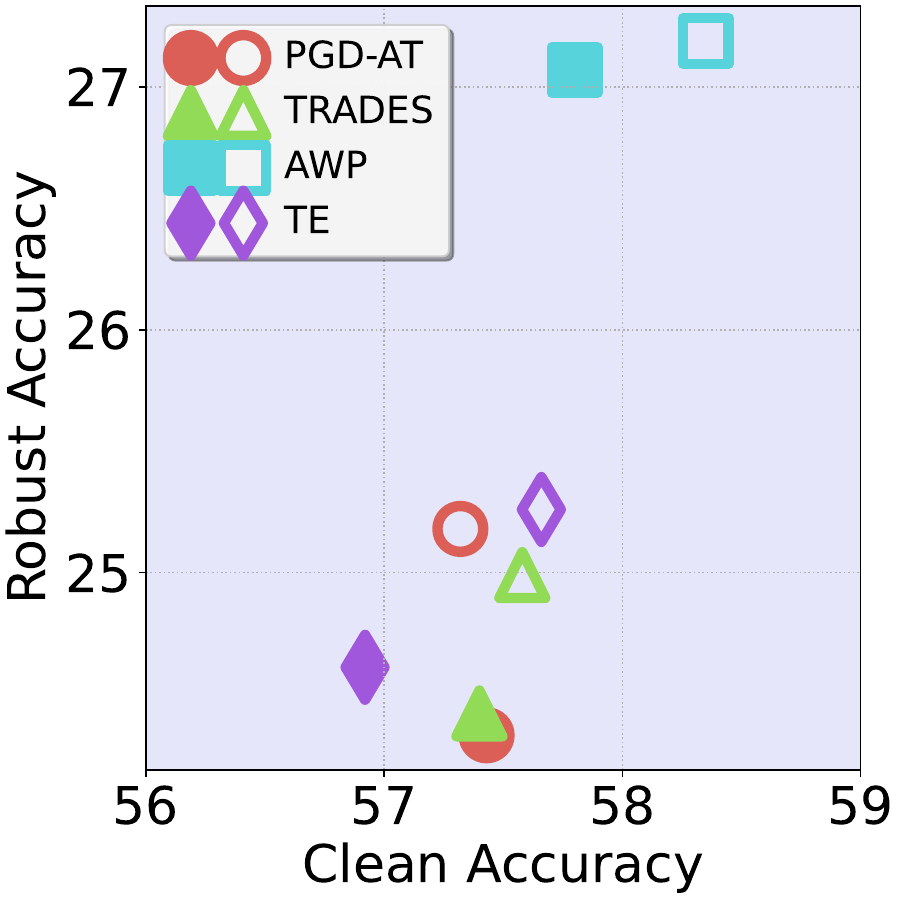}
\caption{Under AA.}
\label{fig:cifar100-resnet-aa} 
\end{subfigure}
\vspace{-10pt}
\caption{Clean accuracy and robust accuracy of ResNet-18 (Solid Markers) and ResNet-18-SR (Hollow Markers) with different training strategies on CIFAR-100. The upper right corner means that the model has the best clean and robust accuracy.}
\label{fig:cifar100-resnet} 
\vspace{-10pt}
\end{figure*}

\subsection{\SysName}

As we show that the IB skip connection can be built with several basic components, in this section we give a reliable example, combined with singular regularization, named \SysName. The overview of \SysName is shown in Figure~\ref{fig:overview}, which contains three main steps, following the design principle of Proposition~\ref{pro:sc}.

For an input image $x$, the first step in \SysName is to use the singular regularization (SR) modules to recalibrate the singular values and singular vectors. To better learn the general layout information from $x$ instead of the fine-grained features, we introduce downsampling operations. With downsampling operations, we can decouple detailed information in the images with the layout information and obtain multiscale inputs. We first downsample the $x$ into four different resolutions, i.e., $x$, $x^{\downarrow}$, $x^{\downarrow\downarrow}$, and $x^{\downarrow\downarrow\downarrow}$, where each $\downarrow$ means a downsample operation. We only use 4 scales, because we consider a trade-off between computational complexity and the effectiveness of decoupling. Specifically, the four different resolutions for the input are 32, 24, 16, and 8, where 32 is the original image resolution. For the input under each scale, we use an independent SR module to regularize its singular values and vectors. For example, we have
\begin{align*}
    x^{\downarrow}_{\mathrm{SR}} = \mathrm{SR}^{\downarrow}(x^\downarrow)
\end{align*}
for $x^\downarrow$.
In the second step, we use the Sigmoid function to compress the information from the SR modules and upsample the output to make them have the same resolution as $x$. After that, we average them to smooth the features of $x$ under different scales, i.e.,
\begin{align*}
    x_\mathrm{avg} = \frac{1}{4}(\varsigma(x_{\mathrm{SR}}) + \varsigma(x^{\downarrow}_{\mathrm{SR}})^{\uparrow} + \varsigma(x^{\downarrow\downarrow}_{\mathrm{SR}})^{\uparrow\uparrow} + \varsigma(x^{\downarrow\downarrow\downarrow}_{\mathrm{SR}})^{\uparrow\uparrow\uparrow}),
\end{align*}
where the $\uparrow$ is an upsampling operation, and $\varsigma$ represent the Sigmoid function. After the upsampling operations, all outputs have the same resolution. In the third step, we adopt convolutional layers $c(\cdot)$ to first extract deep features and then use different convolutional layers, e.g., $p_1(\cdot)$ and $p_2(\cdot)$, to modify the feature dimensions, matching the corresponding dimensions of the classification model's features, i.e., for the $i$-th feature,
\begin{align*}
    f_i = p_i (c(x_\mathrm{avg})).
\end{align*}
We simply add $f_i$ to the corresponding features in the classification model. Note that, in ResNet~\cite{he_deep_2016} or WideResNet~\cite{zagoruyko_wide_2016}, before the first residual block, there is a convolutional layer to convert the input $x$ to high-dimensional features, and we count the outputs of this convolutional layer as the first feature in the classification model. The detailed structures of \SysName can be found in the supplementary materials.

To train \SysName with the classification model, we need to consider constraining both $x_\mathrm{avg}$ and $f_i$ to meet the request that we want \SysName to provide a more compressed representation of $x$ and perform singular regularization simultaneously. More importantly, \SysName should calibrate the misclassification information in the classification model with correct representations. The representation compressing will be executed implicitly due to the IB skip connection. Therefore, we only introduce two loss terms to calibrate misclassification information, i.e., $L_\mathrm{svd}$ and $L_\mathrm{info}$. Specifically,
\begin{align*}
    &L_\mathrm{svd} = \Vert \Sigma_\mathrm{avg} - \Sigma \Vert_2 \\
    + \Vert U_\mathrm{avg} V'_\mathrm{avg} &- UV' \Vert_2 + \Vert x_\mathrm{avg} - x_\mathrm{clean} \Vert_2,
\end{align*}
where $x_\mathrm{avg} = U_\mathrm{avg}\Sigma_\mathrm{avg} V'_\mathrm{avg}$, which is computed from an adversarial example $x_\mathrm{adv}$, and $x_\mathrm{clean} = U\Sigma V'$, which is the corresponding clean image. And,
\begin{align*}
    L_\mathrm{info} = \sum_i \Vert f_i - f_i^\mathrm{clean} \Vert_2,
\end{align*}
where $f_i$ is computed from the adversarial example $x_\mathrm{adv}$, and $f_i^\mathrm{clean}$ is computed from the clean image $x_\mathrm{clean}$. Clearly, $L_\mathrm{svd}$ is proposed to calibrate adversarial information in the singular values and vectors. $L_\mathrm{info}$ is proposed to calibrate adversarial information in the extracted features. Without them, \SysName will not converge to stable conditions. To make the learnable weight matrix in SR orthogonal, we use the penalty term $R_\tau$ in~\cite{zhang_decomposition_2023}. Overall, the training loss can be written as
\begin{align*}
    L = L_\mathrm{ori} + \lambda_1 (L_\mathrm{svd} + L_\mathrm{info}) + \lambda_2 R_\tau,
\end{align*}
where $L_\mathrm{ori}$ is the original training loss, and $\lambda_1$ and $\lambda_2$ are two hyperparameters.

\begin{figure*}[ht]
\centering
\includegraphics[width=1.0\linewidth]{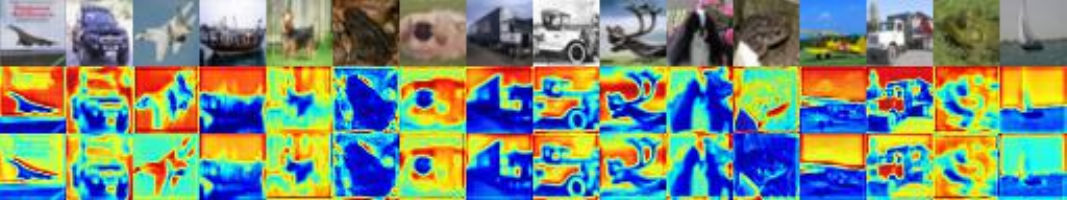}
\vspace{-15pt}
\caption{Heatmap generated by ScoreCam for ResNet-18 and ResNet-18-SR. The first row is clean images. The second row is heatmaps for clean images on ResNet-18. The third row is heatmaps for clean images on ResNet-18-SR.}
\label{fig:scorecam} 
\vspace{-5pt}
\end{figure*}

\begin{table}[ht]
\centering
\begin{adjustbox}{max width=1.0\linewidth}
\begin{tabular}{c|c|c}
 \Xhline{1.5pt}
\textbf{Model} & \textbf{\# of Parameters (M)} & \textbf{Computational Complexity (G)} \\ \hline
ResNet-18 & 11.17 & 0.56 \\ \hline
ResNet-18-SR & 11.35 & 0.73 \\ \hline
WRN-28-10 & 36.48 & 5.25 \\ \hline
WRN-28-10-SR & 36.79 & 5.57 \\
 \Xhline{1.5pt}
\end{tabular}
\end{adjustbox}
\caption{Comparisons of the number of parameters (M, million) and the amount multiply-add operations (G, giga).}
\vspace{-10pt}
\label{tab:cc}
\end{table}

\begin{table*}[ht]
\centering
\begin{adjustbox}{max width=1.0\linewidth}
\begin{tabular}{c|c|c|c|c|c|c}
 \Xhline{1pt}
\textbf{Loss} & $L_\mathrm{CE}$ & $L_\mathrm{svd}$ & $L_\mathrm{info}$ & $L_\mathrm{CE}$+$L_\mathrm{svd}$ & $L_\mathrm{CE}$+$L_\mathrm{info}$ & $L_\mathrm{CE}$+$L_\mathrm{svd}$+$L_\mathrm{info}$ \\ \hline
\textbf{Robust Accuracy} & 52.53 & 82.64 & 83.11 & 52.56 & 52.51 & 52.61 \\ 
 \Xhline{1pt}
\end{tabular}
\end{adjustbox}
\caption{Adaptive attacks for \SysName with different attack loss functions under PGD-20.}
\vspace{-10pt}
\label{tab:ada}
\end{table*}

\section{Experiments}
In this section, we will first introduce the experiment settings. Then, we provide detailed ablation studies to explore the number of $p_i(\cdot)$ and $\lambda_1$. In our main experiments, we compare the results of models trained with different strategies and datasets. We further provide analysis of how the IB skip connections decrease the sensitivity of local perturbation during the attacks. To compare the complexity of the model, we give details of the number of parameters and multiply-add operations in the model. As we introduce a new module, it is necessary to analyze the possible adaptive attacks. We show the results in Section~\ref{sec:aa}. Furthermore, we analyze the singular regularization module with its outputs and performance in the supplementary materials. Besides accuracy under adversarial attacks, we evaluate accuracy under common corruptions in the supplementary materials.

\subsection{Baselines}
We choose two toy image datasets, CIFAR-10~\cite{krizhevsky2009learning} and CIFAR-100~\cite{krizhevsky2009learning}, and one real-world dataset Tiny-Imagenet~\cite{le2015tiny}, and two popular classification models, ResNet-18~\cite{he_deep_2016} and WideResNet-28-10 (WRN-28-10)~\cite{zagoruyko_wide_2016}. For each model, we improve it with \SysName, dubbed ResNet-18-SR and WRN-28-10-SR, respectively. For adversarial training strategies, we consider four different methods, i.e., PGD-AT~\cite{madry_towards_2018}, TRADES~\cite{zhang_theoretically_2019}, AWP~\cite{wu_adversarial_2020}, and TE~\cite{dong_exploring_2022}. Specifically, the adversarial examples generated by each method are constrained under $L_\infty$-norm with settings $\epsilon=8/255$, $\alpha=2/255$, and 10 steps. To train models, we use SGD as an optimizer, with an initial learning rate of 0.1, momentum of 0.9, weight decay of $5e-4$, and batch size of 128. We decrease the learning rate at the 100-th and 150-th epoch by multiplying 0.1. For $\lambda_2$ used in \SysName, we directly use the value in~\cite{zhang_decomposition_2023}, i.e., $\lambda_2=1e-4$. When evaluating the robust models, we consider four 
$L_\infty$-normed attacks in our main paper, and other attacks can be found in supplementary materials. Four representative attacks are the PGD attack~\cite{madry_towards_2018} with cross-entropy loss under 20 and 100 steps (PGD-20 and PGD-100), PGD attack with the C\&W loss~\cite{carlini_towards_2017} under 100 steps (C\&W-100), and AutoAttack (AA)~\cite{croce_reliable_2020}. All these attacks are under $\epsilon=8/255$. Details of \SysName can be found in the supplementary materials.

\subsection{Ablation Study}

We study different configurations in \SysName, including the number of $p_i(\cdot)$, i.e., IB skip connections, and the weight $\lambda_1$ in training loss. The results are obtained on CIFAR-10 with ResNet-18-SR, which is trained under PGD-AT.

In Table~\ref{tab:layer}, we set $\lambda_1 =1.0$ and explore the effectiveness of different skip connections. The results indicate that when we add three $p_i(\cdot)$ to the original model, we can obtain the best trade-off between clean accuracy and robust accuracy. Furthermore, we find that when there are 4 or 5 $p_i(\cdot)$, the model will not converge. Therefore, we do not show them in the table. Too few skip connections or too many skip connections will decrease the performance to a different degree. We think the reasons could be that (1) too few skip connections will only bring information encoding regularization to the top layers, which makes it hard to be consistent in the deeper layers, (2) too many skip connections will disrupt the original intermediate information encoding in the deeper layers, which is hard to be extracted with a simple $p_i(\cdot)$. Although the second reason could be addressed by using a more complex $p_i(\cdot)$, it will increase the computational complexity. In Table~\ref{tab:lambda}, we use three skip connections in the original model and explore the effectiveness of different $\lambda_1$. With increasing $\lambda_1$, both clean accuracy and robust accuracy follow a similar tendency. Therefore, we find that when $\lambda_1 = 20.0$, we can obtain the best clean and robust accuracy simultaneously. In the following experiments, we use three $p_i(\cdot)$ and $\lambda_1=20.0$. 

Furthermore, we find that the information bottleneck will restrict some benign information, hurting clean accuracy, because there exists a trade-off between clean accuracy and robust accuracy under the information bottleneck theory~\cite{alemi_deep_2017}, which is similar to the original explicit optimization with the Lagrange multiplier. To better show such influence in the information bottleneck, we compare $x_\mathrm{avg}$ under different $\lambda_1$ in Figure~\ref{fig:lambda}. The results indicate that a smaller $\lambda_1$ will make the model learn a more compressed representation, containing very less information from the input. On the other hand, a larger $\lambda_1$ will make $x_\mathrm{avg}$ contain more information of the input. It is to say that we implicitly control the strength of the information bottleneck with $\lambda_1$.

\subsection{Main Results}

In this section, we mainly compare the performance of models equipped with \SysName and original models on different datasets. Furthermore, we explore the regional faithfulness of the classification models from the view of sensitivity to perturbation. Specifically, we verify the effectiveness of \SysName from the point of view of information bottleneck theory in the supplementary materials.

In Figure~\ref{fig:cifar10-resnet}, we compare the clean accuracy and robust accuracy under different attacks on CIFAR-10. Because the robust accuracy under different attacks is various, to achieve the best view, we use different y-axis intervals for different attacks. With \SysName, both the clean accuracy and the robust accuracy will increase in general. However, we find that if adversarial training methods adopt additional loss functions besides the cross-entropy loss, such as TE with an MSELoss and TRADES with a KLDiv loss during the training process, the improvement could be less significant than others. The reason we think is that the additional training loss will make the $\lambda_1$ in $L$ we find in Table~\ref{tab:lambda} not optimal. Therefore, for different training strategies, it is better to explore the optimal $\lambda_1$ to achieve the best results. But in our experiments, to keep consistent and universal, we use the same $\lambda_1$ for all training methods.

On the other hand, in Figure~\ref{fig:cifar100-resnet}, we compare clean accuracy and robust accuracy under different attacks on CIFAR-100. In most cases, our \SysName can provide significant improvement for robust accuracy. This is because CIFAR-100 contains more classes than CIFAR-10, resulting in most training strategies not utilizing the full potential of the classification model. However, powerful training strategies, like AWP, already achieve very high clean accuracy and robust accuracy, therefore, the improvement will be less than others. But, \SysName can still improve the clean accuracy. For other results on WRN-28-10 and more, we put them in the supplementary materials.

We notice that most of the previous works in studying adversarial robustness ignore the regional faithfulness of the classification results, i.e., which area of the given image is the more sensitive to noise and will influence the prediction. To analyze the sensitive area of the given images, we adopt ScoreCam~\cite{scorecam} to plot the heatmaps for ResNet-18 and ResNet-18-SR, respectively, which can be found in Figure~\ref{fig:scorecam}. In the heatmaps, the higher temperature means changing such an area will have a more significant influence on the final prediction. Comparing the heatmaps of ResNet-18 and ResNet-18-SR, we find that the original model is more sensitive to the background and \SysName can effectively address such a problem. We think it is because the adversarial perturbation can hide some information in the background area, making the model give incorrect predictions. On the other hand, our \SysName can provide regularization for both singular values and vectors, which is effective in removing background perturbation and decreases the mutual information between the learned presentations and the inputs. Therefore, our method is effective and interpretable.

\subsection{Model Complexity}

We compare the model complexity from two aspects, i.e., the number of model parameters and the amount of multiply-add operations. The results are obtained based on the tool, \textit{ptflops}~\cite{ptflops}. The results in Table~\ref{tab:cc} indicate that \SysName only introduces about 0.8-1.6\% additional parameters, and increases about 0.17-0.32 G additional multiply-add operations, which proves that \SysName is efficient. We notice that when injecting \SysName into WideResNet, it will cause a bigger increase in both the number of parameters and the amount of multiply-add operations. This is because the feature dimensions in WideResNet are higher. Therefore, we need to use more convolutional kernels to match the dimensions, which will introduce more parameters and multiply-add operations. In summary, \SysName is efficient. 

\subsection{Adaptive Attacks}
\label{sec:aa}

As we introduce a new module into the original model, it is important to learn the potential adaptive attacks. We consider building adaptive attacks based on the loss terms $L_\mathrm{svd}$ and $L_\mathrm{info}$, because they will influence the features extracted by \SysName. In Table~\ref{tab:ada}, we show the robust accuracy under PGD-20 attacks using different loss terms, where $L_\mathrm{CE}$ is the cross-entropy loss, i.e., the default loss in PGD. The results indicate that even under adaptive attacks, the model still obtains a high robust accuracy. It is because we only use \SysName to restrict intermediate features in the original models, instead of directly using the features from \SysName. The visualization results can be found in the supplementary materials.

\section{Conclusion}

In this paper, we study adversarial examples in the view of singular value decomposition. We find that the position of adversarial information in singular values and vectors is related to the attacks, instead of the training methods of the models. Based on this observation, we propose a new plug-in module, \SysName, for classification models. When designing \SysName, we innovatively introduce information bottleneck theory. By proving a general structure can fulfill the implicit information bottleneck regularization, we design \SysName in the same way. Through comprehensive experiments, we show the effectiveness of our method and prove its efficiency. We also give some interpretable evidence to verify the effectiveness of \SysName. We hope that our work can provide a new perspective on robust model design and inspire others. Moreover, addressing the trade-off between clean accuracy and robust accuracy is another future work.

{
    \small
    \bibliographystyle{ieeenat_fullname}
    \bibliography{bib}

\begin{thebibliography}{28}
\providecommand{\natexlab}[1]{#1}
\providecommand{\url}[1]{\texttt{#1}}
\expandafter\ifx\csname urlstyle\endcsname\relax
  \providecommand{\doi}[1]{doi: #1}\else
  \providecommand{\doi}{doi: \begingroup \urlstyle{rm}\Url}\fi

\bibitem[Alemi et~al.(2017)Alemi, Fischer, Dillon, and Murphy]{alemi_deep_2017}
Alexander~A. Alemi, Ian Fischer, Joshua~V. Dillon, and Kevin Murphy.
\newblock Deep {Variational} {Information} {Bottleneck}.
\newblock In \emph{Proc. of the {ICLR}}, 2017.

\bibitem[Athalye et~al.(2018)Athalye, Carlini, and Wagner]{athalye_obfuscated_2018}
Anish Athalye, Nicholas Carlini, and David~A. Wagner.
\newblock Obfuscated {Gradients} {Give} a {False} {Sense} of {Security}: {Circumventing} {Defenses} to {Adversarial} {Examples}.
\newblock In \emph{Proc. of the {ICML}}, pages 274--283, 2018.

\bibitem[Bai et~al.(2021)Bai, Mei, Yuille, and Xie]{bai}
Yutong Bai, Jieru Mei, Alan~L. Yuille, and Cihang Xie.
\newblock Are transformers more robust than cnns?
\newblock In \emph{Proc. of the NeurIPS}, pages 26831--26843, 2021.

\bibitem[Carlini and Wagner(2017)]{carlini_towards_2017}
Nicholas Carlini and David Wagner.
\newblock Towards {Evaluating} the {Robustness} of {Neural} {Networks}.
\newblock In \emph{Proc. of the {SP}}, pages 39--57, 2017.

\bibitem[Croce and Hein(2020)]{croce_reliable_2020}
Francesco Croce and Matthias Hein.
\newblock Reliable evaluation of adversarial robustness with an ensemble of diverse parameter-free attacks.
\newblock In \emph{Proc. of the {ICML}}, pages 2206--2216, 2020.

\bibitem[Dong et~al.(2022)Dong, Xu, Yang, Pang, Deng, Su, and Zhu]{dong_exploring_2022}
Yinpeng Dong, Ke Xu, Xiao Yang, Tianyu Pang, Zhijie Deng, Hang Su, and Jun Zhu.
\newblock Exploring {Memorization} in {Adversarial} {Training}.
\newblock In \emph{Proc. of the {ICLR}}, 2022.

\bibitem[Dosovitskiy et~al.(2021)Dosovitskiy, Beyer, Kolesnikov, Weissenborn, Zhai, Unterthiner, Dehghani, Minderer, Heigold, Gelly, Uszkoreit, and Houlsby]{vit}
Alexey Dosovitskiy, Lucas Beyer, Alexander Kolesnikov, Dirk Weissenborn, Xiaohua Zhai, Thomas Unterthiner, Mostafa Dehghani, Matthias Minderer, Georg Heigold, Sylvain Gelly, Jakob Uszkoreit, and Neil Houlsby.
\newblock An image is worth 16x16 words: Transformers for image recognition at scale.
\newblock In \emph{Proc. of the ICLR}, 2021.

\bibitem[Federici et~al.(2020)Federici, Dutta, Forr{\'{e}}, Kushman, and Akata]{mvib}
Marco Federici, Anjan Dutta, Patrick Forr{\'{e}}, Nate Kushman, and Zeynep Akata.
\newblock Learning robust representations via multi-view information bottleneck.
\newblock In \emph{Proc. of the ICLR}, 2020.

\bibitem[Ge et~al.(2023)Ge, Li, Han, Zhu, and Long]{ge_advancing_2023}
Yao Ge, Yun Li, Keji Han, Junyi Zhu, and Xianzhong Long.
\newblock Advancing {Example} {Exploitation} {Can} {Alleviate} {Critical} {Challenges} in {Adversarial} {Training}.
\newblock In \emph{Proc. of the {ICCV}}, pages 145--154, 2023.

\bibitem[Goodfellow et~al.(2015)Goodfellow, Shlens, and Szegedy]{goodfellow_explaining_2015}
Ian~J. Goodfellow, Jonathon Shlens, and Christian Szegedy.
\newblock Explaining and {Harnessing} {Adversarial} {Examples}.
\newblock In \emph{Proc. of the {ICLR}}, 2015.

\bibitem[He et~al.(2016)He, Zhang, Ren, and Sun]{he_deep_2016}
Kaiming He, Xiangyu Zhang, Shaoqing Ren, and Jian Sun.
\newblock Deep {Residual} {Learning} for {Image} {Recognition}.
\newblock In \emph{Proc. of the {CVPR}}, pages 770--778, 2016.

\bibitem[Hendrycks and Dietterich(2019)]{hendrycks2019robustness}
Dan Hendrycks and Thomas Dietterich.
\newblock Benchmarking neural network robustness to common corruptions and perturbations.
\newblock \emph{Proc. of the ICLR}, 2019.

\bibitem[Huang et~al.(2023)Huang, Lu, Deb, and Boddeti]{huang}
Shihua Huang, Zhichao Lu, Kalyanmoy Deb, and Vishnu~Naresh Boddeti.
\newblock Revisiting residual networks for adversarial robustness.
\newblock In \emph{Proc. of the CVPR}, pages 8202--8211, 2023.

\bibitem[Ilyas et~al.(2019)Ilyas, Santurkar, Tsipras, Engstrom, Tran, and Madry]{ilyas_adversarial_2019}
Andrew Ilyas, Shibani Santurkar, Dimitris Tsipras, Logan Engstrom, Brandon Tran, and Aleksander Madry.
\newblock Adversarial {Examples} {Are} {Not} {Bugs}, {They} {Are} {Features}.
\newblock In \emph{Proc. of the {NeurIPS}}, pages 125--136, 2019.

\bibitem[Krizhevsky et~al.(2009)Krizhevsky, Hinton, et~al.]{krizhevsky2009learning}
Alex Krizhevsky, Geoffrey Hinton, et~al.
\newblock Learning multiple layers of features from tiny images.
\newblock 2009.

\bibitem[Le and Yang(2015)]{le2015tiny}
Ya Le and Xuan Yang.
\newblock Tiny imagenet visual recognition challenge.
\newblock \emph{CS 231N}, 7\penalty0 (7):\penalty0 3, 2015.

\bibitem[Lee et~al.(2021)Lee, Choi, Mok, and Yoon]{lee_reducing_2021}
Jungbeom Lee, Jooyoung Choi, Jisoo Mok, and Sungroh Yoon.
\newblock Reducing {Information} {Bottleneck} for {Weakly} {Supervised} {Semantic} {Segmentation}.
\newblock In \emph{Proc. of the {NeurIPS}}, pages 27408--27421, 2021.

\bibitem[Madry et~al.(2018)Madry, Makelov, Schmidt, Tsipras, and Vladu]{madry_towards_2018}
Aleksander Madry, Aleksandar Makelov, Ludwig Schmidt, Dimitris Tsipras, and Adrian Vladu.
\newblock Towards {Deep} {Learning} {Models} {Resistant} to {Adversarial} {Attacks}.
\newblock In \emph{Proc. of the {ICLR}}, 2018.

\bibitem[Rice et~al.(2020)Rice, Wong, and Kolter]{rice_overfitting_2020}
Leslie Rice, Eric Wong, and J.~Zico Kolter.
\newblock Overfitting in adversarially robust deep learning.
\newblock In \emph{Proc. of the {ICML}}, pages 8093--8104, 2020.

\bibitem[Sovrasov(2018-2023)]{ptflops}
Vladislav Sovrasov.
\newblock ptflops: a flops counting tool for neural networks in pytorch framework, 2018-2023.

\bibitem[Wall et~al.(2003)Wall, Rechtsteiner, and Rocha]{wall_singular_2003}
Michael~E. Wall, Andreas Rechtsteiner, and Luis~M. Rocha.
\newblock Singular {Value} {Decomposition} and {Principal} {Component} {Analysis}.
\newblock In \emph{A {Practical} {Approach} to {Microarray} {Data} {Analysis}}, pages 91--109. Springer US, 2003.

\bibitem[Wang et~al.(2020{\natexlab{a}})Wang, Wang, Du, Yang, Zhang, Ding, Mardziel, and Hu]{scorecam}
Haofan Wang, Zifan Wang, Mengnan Du, Fan Yang, Zijian Zhang, Sirui Ding, Piotr Mardziel, and Xia Hu.
\newblock Score-cam: Score-weighted visual explanations for convolutional neural networks.
\newblock In \emph{Proc. of the CVPR Workshop}, pages 111--119, 2020{\natexlab{a}}.

\bibitem[Wang et~al.(2020{\natexlab{b}})Wang, Zou, Yi, Bailey, Ma, and Gu]{wang_improving_2020}
Yisen Wang, Difan Zou, Jinfeng Yi, James Bailey, Xingjun Ma, and Quanquan Gu.
\newblock Improving {Adversarial} {Robustness} {Requires} {Revisiting} {Misclassified} {Examples}.
\newblock In \emph{Proc. of the {ICLR}}, 2020{\natexlab{b}}.

\bibitem[Wu et~al.(2020)Wu, Xia, and Wang]{wu_adversarial_2020}
Dongxian Wu, Shu-Tao Xia, and Yisen Wang.
\newblock Adversarial {Weight} {Perturbation} {Helps} {Robust} {Generalization}.
\newblock In \emph{Proc. of the {NeurIPS}}, 2020.

\bibitem[Xie et~al.(2019)Xie, Wu, Maaten, Yuille, and He]{xie_feature_2019}
Cihang Xie, Yuxin Wu, Laurens van~der Maaten, Alan~L Yuille, and Kaiming He.
\newblock Feature {Denoising} for {Improving} {Adversarial} {Robustness}.
\newblock In \emph{Proc. of the {CVPR}}, pages 501--509, 2019.

\bibitem[Zagoruyko and Komodakis(2016)]{zagoruyko_wide_2016}
Sergey Zagoruyko and Nikos Komodakis.
\newblock Wide {Residual} {Networks}.
\newblock In \emph{Proc. of the {BMVC}}, 2016.

\bibitem[Zhang et~al.(2019)Zhang, Yu, Jiao, Xing, Ghaoui, and Jordan]{zhang_theoretically_2019}
Hongyang Zhang, Yaodong Yu, Jiantao Jiao, Eric~P. Xing, Laurent~El Ghaoui, and Michael~I. Jordan.
\newblock Theoretically {Principled} {Trade}-off between {Robustness} and {Accuracy}.
\newblock In \emph{Proc. of the {ICML}}, pages 7472--7482, 2019.

\bibitem[Zhang et~al.(2023)Zhang, Huang, Zhou, Li, and Zhao]{zhang_decomposition_2023}
Jinghao Zhang, Jie Huang, Man Zhou, Chongyi Li, and Feng Zhao.
\newblock Decomposition {Ascribed} {Synergistic} {Learning} for {Unified} {Image} {Restoration}.
\newblock \emph{CoRR}, abs/2308.00759, 2023.

\end{thebibliography}
}

\appendix

\begin{table*}[h]
\centering
\begin{adjustbox}{max width=1.0\linewidth}
\begin{tabular}{c|c|c|c|c|c|c|c}
 \Xhline{2pt}
\textbf{Model} & \textbf{\# of Parameters (M)} & \textbf{Computational Complexity (G)} & \textbf{Clean Accuracy} & \textbf{PGD-20} & \textbf{PGD-100} & \textbf{C\&W-100} & \textbf{AA} \\ \hline
\textbf{WRN-28-10} & 36.48 & 5.25 & 86.78 & 53.68 & 53.48 & 53.22 & 50.92 \\ \hline
\textbf{WRN-28-10-SR} & 36.79 & 5.57 & 86.74 & 54.30 & 53.90 & 52.99 & 50.94 \\ \hline
\textbf{WRN-28-10-SR-L} & 37.89 & 6.70 & 86.66 & 54.26 & 53.93 & 53.40 & 51.22 \\ \hline
\textbf{WRN-28-10-SR-XL} & 38.81 & 7.65 & \textbf{86.91} & \textbf{54.55} & \textbf{54.25} & \textbf{53.90} & \textbf{51.53} \\
 \Xhline{2pt}
\end{tabular}
\end{adjustbox}
\caption{Comparisons between different WRN-28-10 structures.}
\vspace{-10pt}
\label{tab:wrn}
\end{table*}

\begin{figure*}[ht]
\centering
\begin{subfigure}[b]{0.45\linewidth}
\centering
\includegraphics[width=\linewidth]{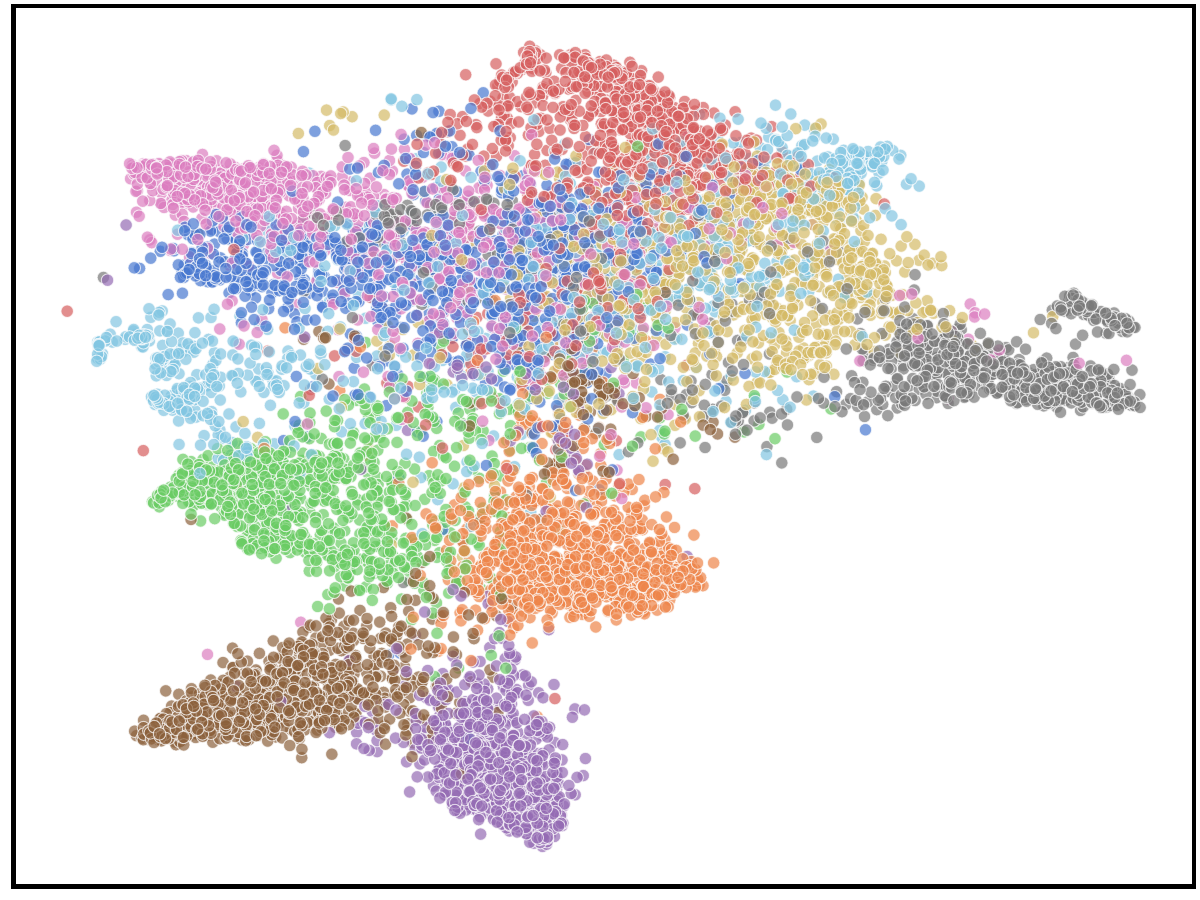}
\caption{ResNet-18}
\label{fig:tsne-ori} 
\end{subfigure}
\begin{subfigure}[b]{0.45\linewidth}
\centering
\includegraphics[width=\linewidth,height=168pt]{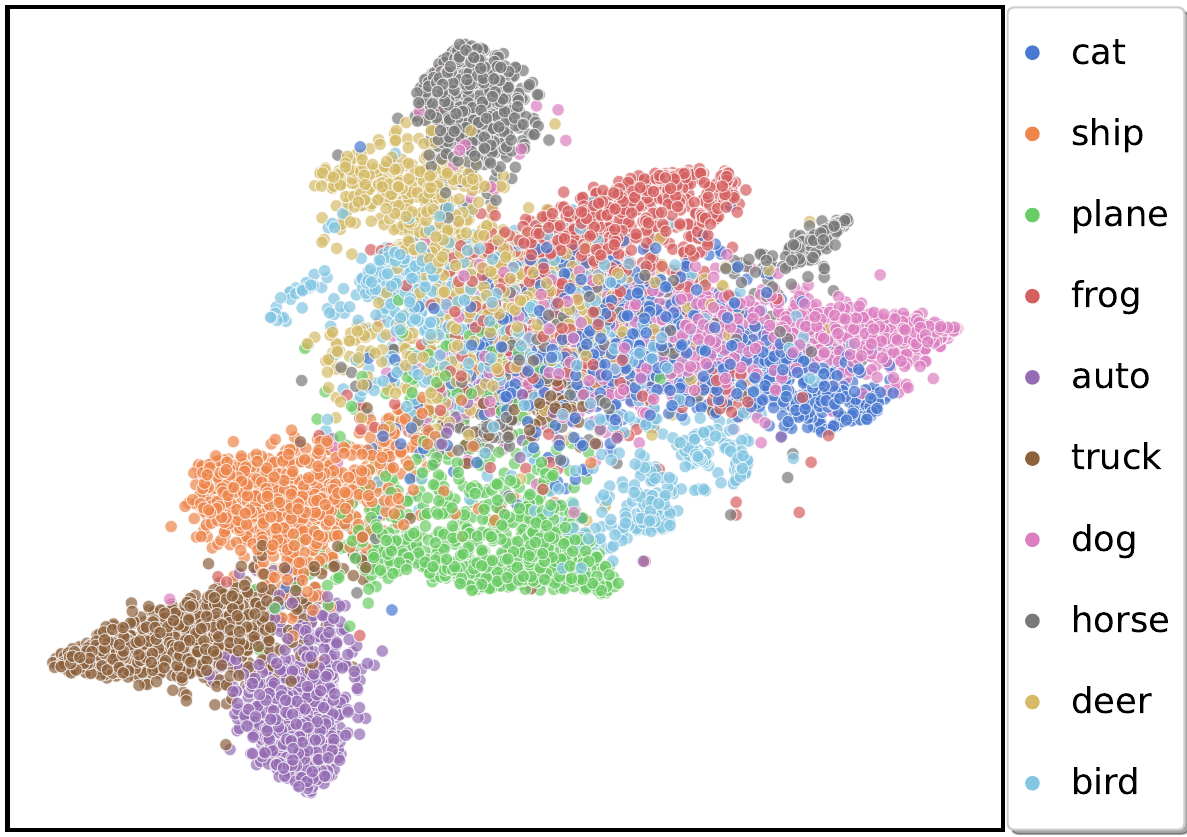}
\caption{ResNet-18-SR}
\label{fig:tsne-sr} 
\end{subfigure}
\caption{T-SNE results for ResNet-18 and ResNet-18-SR on CIFAR-10, respectively. Best viewed in color.}
\vspace{-10pt}
\label{fig:tsne} 
\end{figure*}

\section{Proof}

\begin{proposition}
If functional $\mathcal{G}(\cdot; \tau)$ is equipped with a double-sided saturating activation function, e.g., Sigmoid, there exists $g(\cdot;\tau)\in \mathcal{G}(\cdot; \tau)$, which can be used as an IB skip connection.
\end{proposition}

\begin{proof}
To prove the existence of such a $g(\cdot;\tau)$ meeting the condition that $I(z, x; \theta, \tau) \le I(f(x;\theta), x; \theta)$, we only need to construct such a function.

Let $z_1 = f(x;\theta)$, $z_2 = g(x;\tau)$, and $z=z_1 + z_2$. We can rewrite $I(z, x; \theta, \tau)$ and $I(f(x;\theta), x; \theta)$ as:
\begin{align*}
    I(z, & x) = I(z_1 + z_2 , x) = H(z_1 + z_2) - H(z_1 + z_2 | x)\\
    & I(f(x;\theta), x) = I(z_1, x) = H(z_1) - H(z_1|x),
\end{align*}
where $H(\cdot)$ is the entropy. We aim to construct $g(x;\tau)$, fulfilling 
\begin{align*}
   H(z_1 + z_2) - H(z_1 + z_2 | x) \le H(z_1) - H(z_1|x).
\end{align*} 
Therefore, the construction is inspired from two aspects. First, $g(x;\tau)$ should decrease the uncertainty of $f(x;\theta)$, which loosely makes $H(z_1 + z_2) \le H(z_1) + \epsilon$. Second, $g(x;\tau)$ should increase the information of $f(x;\theta)$ extracted from $x$, which makes $H(z_1 + z_2 | x) \ge H(z_1|x)$. Clearly, for the second aspect, adding a new random variable to the conditional entropy will increase the value, i.e., $H(z_1 + z_2 | x) \ge H(z_1|x)$ holds in general.

To decrease the uncertainty of $f(x;\theta)$ by adding $g(x;\tau)$ to it, the construction is designed as follows:
\begin{align*}
    g(x;\tau) &= g_2 \circ \varsigma \circ g_1 (x),\\
    \varsigma (x) &= \frac{1}{1+e^{-x}},
\end{align*}
where $g_1$ and $g_2$ are two linear transformations. Then, we describe the reason why this structure can make $H(z_1 + z_2) \le H(z_1)$. 

Note that both $f(x;\theta)$ and $g(x;\tau)$ are calculated on the basis of $x$, which means that they are not independent. So, $g_1$ and $g_2$ will keep information of $x$ with the linear transformations, making $z_1$ and $z_2$ share redundant information of $x$. On the other hand, $\varsigma$ maps the values to $(0, 1)$, which reduces $H(\varsigma \circ g_1(x))$ and let $g(x;\tau)$ learn a compressed representation. With a proper supervision signal, for example, cross-entropy loss between predictions and ground-truth labels, $g(x;\tau)$ will be an IB skip connection.

Therefore, we find a possible case of $g$, which can be trained as an IB skip connection.
\end{proof}

\section{Details of \SysName}
\label{sec:detail}

We will introduce the details of the \SysName in three parts. First, we introduce the details of SR~\cite{zhang_decomposition_2023}. To modulate singular vectors, we adopt one convolutional layer, whose weight is an orthogonal matrix. After modulating singular vectors, we adopt fast Fourier transforms to transfer the input into the Fourier domain. Then, we use a convolutional layer to modulate the signal and transfer it from the Fourier domain to the original domain. Finally, we fuse the outputs with one convolutional layer and use a batch normalization layer to stabilize the output. Specifically, the kernel size and stride are 1 for all convolutional layers in SR. The input channels and output channels of the convolutional layers are 3 for RGB images. The number of channels of the orthogonal kernel is 12.

Second, we introduce the details of $c(\cdot)$. The design of $c(\cdot)$ is straightforward, in which we want to extract deep features with several convolutional layers. Considering the computational complexity, we adopt three convolutional layers with increasing channels, and the batch normalization layer and ReLU activation layer are added after each convolutional layer. For all convolutional layers, the kernel size is 3 and the stride is 1. The numbers of output channels are 16, 32, and 64. 

Finally, we introduce the details of $p_i(\cdot)$. Considering $p_i(\cdot)$ is to modify the dimension of the extracted features of $c(\cdot)$, we use one simple convolutional layer, whose kernel size is 3 and stride is 1, to increase or decrease channels to match the target channel amount. If needed, we add a downsampling operation to match the width and height dimensions of the feature.

\section{Experiment Results}

In this section, we present other results from our experiments to better demonstrate the superiority of our method.

\subsection{\SysName on Real-World Dataset}

In our main paper, we only show the results on two toy datasets. In this section, we will compare the results on a real-world dataset, Tiny-Imagenet. We train the models with FGSM-AT, and evaluate models under various $l_\infty$-norm attacks. Especially, because the image resolution is 64, we increase the downsample resolution to 48, 24, and 16.

\begin{table}[h]
\centering
\begin{adjustbox}{max width=1.0\linewidth}
\begin{tabular}{c|c|c|c|c|c}
 \Xhline{2pt}
\textbf{Model} & \textbf{Clean Accuracy} & \textbf{PGD-20} & \textbf{PGD-100} & \textbf{C\&W-100} & \textbf{AA} \\ \hline
\textbf{ResNet-18} & 42.47 & 33.92 & 33.91 & 32.14 & 31.85 \\ \hline
\textbf{ResNet-18-SR} & \textbf{43.52} & \textbf{34.96} & \textbf{34.97} & \textbf{33.25} & \textbf{32.85} \\
 \Xhline{2pt}
\end{tabular}
\end{adjustbox}
\caption{Evaluation on a real-world dataset, Tiny-Imagenet. The attack budget $\epsilon=2/255$.}
\vspace{-10pt}
\label{tab:tiny}
\end{table}

In Table~\ref{tab:tiny}, we compare the results of ResNet-18 and ResNet-18-SR. Clearly, with \SysName, the model will obtain both higher clean accuracy and robust accuracy. The improvement is about 1\%. The results indicate that our method works for real-world datasets.

\subsection{\SysName on Different Model Structure}

In this section, we compare \SysName on different model structures and evaluate the different designs of $c(\cdot)$ in \SysName. In Table~\ref{tab:wrn}, we first compare the results of WRN-28-10 and WRN-28-10-SR on CIFAR-10, trained with PGD-AT. The results indicate that when equipped with \SysName, the model will obtain higher robust accuracy under different attacks. It is to say that \SysName is general for different model architectures. On the other hand, because WRN-28-10 extracts deeper features than ResNet-18, we consider different settings of $c(\cdot)$ to study the influence of the extracted features in \SysName. Specifically, $c(\cdot)$ in WRN-28-10-SR is described in Section~\ref{sec:detail}. For WRN-28-10-SR-L and WRN-28-10-SR-XL, we consider using more output channels. In detail, we use two cascaded basic blocks as $c(\cdot)$ for WRN-28-10-SR-L and four cascaded basic blocks as $c(\cdot)$ for WRN-28-10-SR-XL. The total number of parameters and computational complexity are shown in Table~\ref{tab:wrn}. From the results, we can find that when we use a more complex $c(\cdot)$, robust accuracy will significantly increase, especially for C\&W-100 and AA. This is because a more complex $c(\cdot)$ will extract features to better match the intermediate representations of the classification model, which will help the classification model better converge. Therefore, for a larger classification model, the better choice of $c(\cdot)$ is to use a small deep convolutional model.

\subsection{Visualization of Feature Representation}

To prove the effectiveness of \SysName in feature learning, we show the T-SNE results to study both representation distributions and compressions, just like the previous work~\cite{mvib}. The visualization results in Figure~\ref{fig:tsne} indicate that with \SysName, the model will learn more compressed representations (axes of both plots are zoomed to the same space). Furthermore, the features are better separated under the T-SNE views. Therefore, \SysName works like implicitly adding an information bottleneck to the training process.

\subsection{Robust Accuracy under $l_2$ Attacks}

We show robust accuracy under various $l_2$-normed attacks to prove that our method is general for adversarial attacks under different norms. We consider PGD attacks and C\&W attacks. For both, we set $\epsilon=0.5$ and $\alpha=0.1$. In Table~\ref{tab:l2}, the results indicate that \SysName can improve the robustness under $l_2$-normed attacks, which means that our proposed method is general.

\begin{table}[h]
\centering
\begin{adjustbox}{max width=1.0\linewidth}
\begin{tabular}{c|c|c|c|c}
 \Xhline{2pt}
\textbf{Model} & \textbf{Clean Accuracy} & \textbf{PGD-20} & \textbf{PGD-100} & \textbf{C\&W-100} \\ \hline
\textbf{ResNet-18} & \textbf{83.28} & 62.38 & 61.97 & 60.08 \\ \hline
\textbf{ResNet-18-SR} & 83.27 & \textbf{62.88} & \textbf{62.44} & \textbf{60.65} \\
 \Xhline{2pt}
\end{tabular}
\end{adjustbox}
\caption{Robust accuracy under $l_2$-normed attacks.}
\vspace{-10pt}
\label{tab:l2}
\end{table}

\begin{figure*}[ht]
\centering
\includegraphics[width=1.0\linewidth]{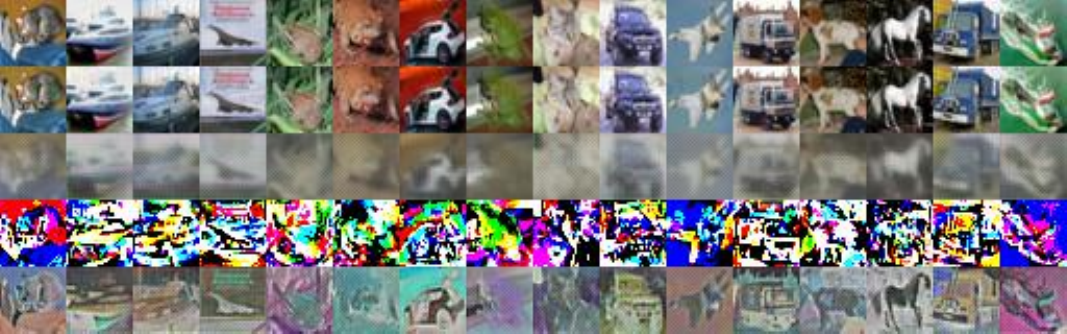}
\vspace{-15pt}
\caption{Visualization results of $x$, $x_\mathrm{adv}$, and $x_\mathrm{avg}$. The first row is clean images $x$. The second row shows the adversarial examples $x_\mathrm{adv}$. The third row represents $x_\mathrm{avg}$, whose input is $x_\mathrm{adv}$. The fourth row is $x - x_\mathrm{adv}$. The fifth row is $x_\mathrm{avg} - x_\mathrm{adv}$.}
\label{fig:img} 
\vspace{-10pt}
\end{figure*}

\begin{table}[ht]
\centering
\begin{adjustbox}{max width=1.0\linewidth}
\begin{tabular}{c|c|c|c|c|c}
 \Xhline{1.5pt}
\textbf{Input} & \textbf{Clean Accuracy} & \textbf{PGD-20} & \textbf{PGD-100} & \textbf{C\&W-100} & \textbf{AA} \\ \hline
$x$ or $x_\mathrm{adv}$ & \textbf{83.28} & 51.28 & 50.97 & 49.53 & 47.43 \\ \hline
$x_\mathrm{avg}$ & 80.95 & \textbf{52.80} & \textbf{52.76} & \textbf{61.44} & \textbf{54.65} \\ 
 \Xhline{1.5pt}
\end{tabular}
\end{adjustbox}
\caption{The clean accuracy and robust accuracy under the input of $x$, $x_\mathrm{adv}$, and $x_\mathrm{avg}$, respectively.}
\vspace{-15pt}
\label{tab:is}
\end{table}

\subsection{Are Singular Values and Vectors Corrected?}
\label{sec:sr}

To better study how the singular regularization works in the model inference phase, we show the visualization results in Figure~\ref{fig:img}. Additionally, we normalize both $x - x_\mathrm{adv}$ and $x_\mathrm{avg} - x_\mathrm{adv}$ to obtain the best views. From the plots, we can find that the singular regularization will remove some details from the input images, which correspond to the small singular values in the SVD. In other words, singular regularization keeps the main layout information and removes the additional unimportant details, which achieves information compression. This observation is aligned with our design from the perspective of information bottleneck theory.

On the other hand, we explore how much adversarial information has been removed in $x_\mathrm{avg}$. To obtain the results, we first isolate $x_\mathrm{avg}$ from the model training process. Then, we only train the SR module on the adversarial examples generated by a ResNet model during its adversarial training process. Therefore, $x_\mathrm{avg}$ will be a regularized input. When evaluating the removed adversarial information, we first generate adversarial examples $x_\mathrm{adv}$ with a ResNet model. Then we calculate $x_\mathrm{avg}$ based on $x_\mathrm{adv}$ and input $x_\mathrm{avg}$ to the ResNet model. That is to say, we consider a grey-box attack, where the adversary has no information about the SR module. As we only care about the adversarial information removed by SR module, this grey-box attack is reasonable. In Table~\ref{tab:is}, we compare clean accuracy when the input is $x$ and $x_\mathrm{avg}$, which is calculated on $x$. Then, we compare robust accuracy when the input is $x_\mathrm{adv}$ and $x_\mathrm{avg}$, which is calculated on $x_\mathrm{adv}$. For clean accuracy, we find the removed details will slightly decrease it. It is because when predicting the image's class, the model will use all information in the image. Removing benign information will have a negative impact on the model. For PGD-20 and PGD-100, we find robust accuracy increases by about 2\%. For C\&W-100 and AA, the improvement is more significant. These results are aligned with the observation in Figure 1 in the main paper, where we find that adversarial examples generated by C\&W-100 and AA are much easier to be purified by the singular regularization. Therefore, our analysis proves that the singular regularization will modify the singular values and vectors and remove the adversarial attack information.

\subsection{Accuracy under Common Corruptions}

In addition to adversarial perturbation, we consider the influence of common corruptions~\cite{hendrycks2019robustness}. Similarly, we isolate the SR modules during adversarial training and use them again in the inference phase to avoid the trade-off between information compression and representation expression. In Table~\ref{tab:cn}, we compare the accuracy of ResNet-18 and ResNet-18-SR under different common corruptions. The results indicate that with singular regularization, the proposed method can further defend against common corruptions, even the model is trained on adversarial examples.

\begin{table*}[h]
\centering
\begin{adjustbox}{max width=1.0\linewidth}
\begin{tabular}{c|c|c|c|c|c|c|c|c|c|c|c|c}
 \Xhline{2pt}
\textbf{Type} & \textbf{ShotNoise} & \textbf{GaussianNoise} & \textbf{ImpulseNoise} & \textbf{MotionBlur} & \textbf{GlassBlur} & \textbf{DefocusBlur} & \textbf{ZoomBlur} & \textbf{JPEGCompression} & \textbf{ElasticTransform} & \textbf{Pixelate} & \textbf{Fog} & \textbf{Frost} \\ \hline
\textbf{ResNet-18} & 75.59 & 75.17 & 66.76 & 67.97 & 73.31 & 71.22 & 73.94 & 80.28 & 75.38 & 78.33 & 28.03 & 68.20 \\ \hline
\textbf{ResNet-18-SR} & \textbf{77.60} & \textbf{77.11} & \textbf{68.67} & \textbf{70.15} & \textbf{74.25} & \textbf{72.55} & \textbf{75.35} & \textbf{80.91} & \textbf{75.97} & \textbf{78.93} & \textbf{30.33} & \textbf{69.36} \\
 \Xhline{2pt}
\end{tabular}
\end{adjustbox}
\caption{Accuracy under various common corruptions.}
\vspace{-10pt}
\label{tab:cn}
\end{table*}

\subsection{Additional Visualization Results}

We provide more visualization results in Figure~\ref{fig:img_supp}. To better show the singular regularization effects, we compare the $x_\mathrm{avg}$ under adversarial examples generated by PGD-20 and C\&W-100, respectively. From the results, we can find that the SR module is stable under different attacks. Furthermore, we explore whether adaptive attacks will influence the performance of our SR module. In Figure~\ref{fig:img_supp_ada}, we compare the $x_\mathrm{avg}$ under different adaptive attacks. The results indicate that the SR module can maintain the calibration function even under our designed adaptive attacks. 

\begin{figure*}[h]
\centering
\includegraphics[width=1.0\linewidth]{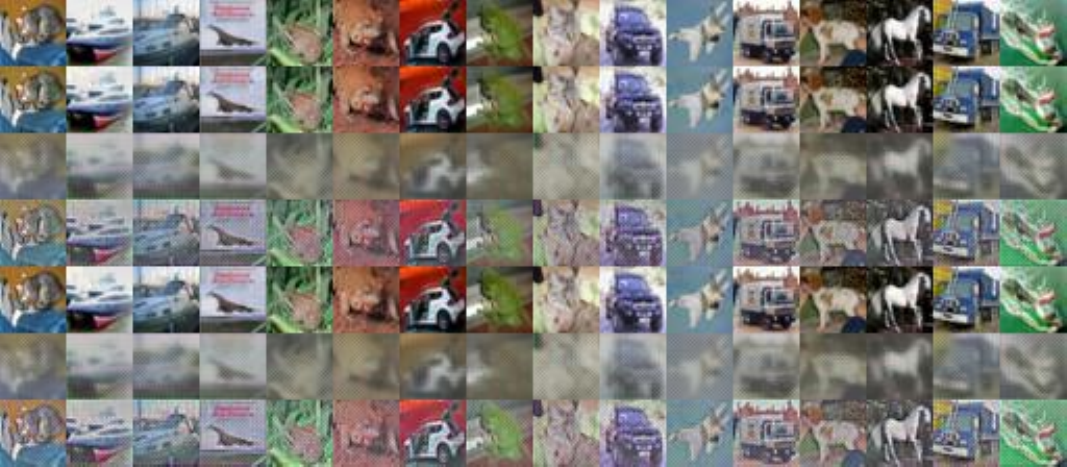}
\vspace{-5pt}
\caption{Visualization results of $x$, $x_\mathrm{adv}$, and $x_\mathrm{avg}$. The first row is clean images $x$. The second row shows the adversarial examples $x_\mathrm{adv}$ generated by PGD-20. The third row represents $x_\mathrm{avg}$, whose input is the second row. The fourth row is $x_\mathrm{adv} - x_\mathrm{avg}$ of the two above rows. The fifth row is the adversarial examples $x_\mathrm{adv}$ generated by C\&W-100. The sixth row represents $x_\mathrm{avg}$, whose input is the fifth row. The seventh row is $x_\mathrm{adv} - x_\mathrm{avg}$ of the above two rows.}
\label{fig:img_supp} 
\vspace{-10pt}
\end{figure*}

\begin{figure*}[h]
\centering
\includegraphics[width=1.0\linewidth]{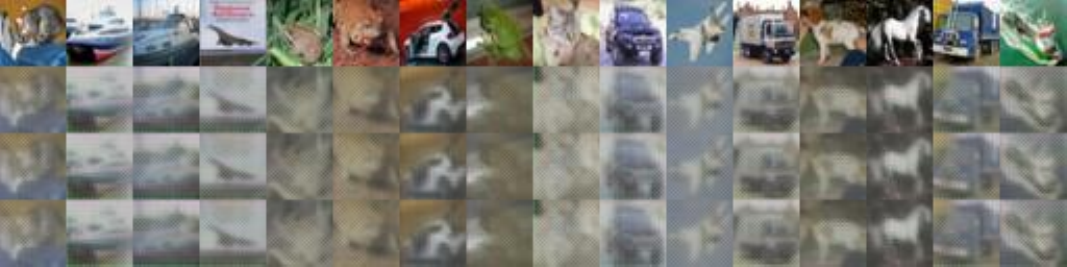}
\vspace{-5pt}
\caption{Visualization results of $x$, $x_\mathrm{adv}$, and $x_\mathrm{avg}$. The first row is clean images $x$. The second row shows $x_\mathrm{avg}$, whose input is adversarial examples generated by PGD-20 with $L_\mathrm{CE}$. The third row is $x_\mathrm{avg}$, whose input is adversarial examples generated by PGD-20 with $L_\mathrm{info}$. The second row represents $x_\mathrm{avg}$, whose input is adversarial examples generated by PGD-20 with $L_\mathrm{svd}$.}
\label{fig:img_supp_ada} 
\vspace{-10pt}
\end{figure*}

\end{document}